\documentclass{article}
\usepackage{fullpage,times}
\usepackage{parskip}
\usepackage[utf8]{inputenc} 
\usepackage[T1]{fontenc}    
\usepackage{url}            
\usepackage{booktabs}       
\usepackage{amsfonts}       
\usepackage{nicefrac}       
\usepackage{natbib}
\usepackage{graphicx}
\usepackage{subfigure}

\usepackage{amsmath, amsthm, amssymb, mathtools, enumerate}
\usepackage{adjustbox}
\usepackage{xspace}
\usepackage{color}
\usepackage{enumitem}
\usepackage{comment}
\usepackage{bm}
\usepackage{apptools}
\usepackage[page, header]{appendix}
\usepackage{titletoc}
\usepackage{array}
\usepackage{dsfont}
\usepackage{multirow}
\usepackage{hhline}
\usepackage{tabu}
\usepackage{bbm}

\newcount\Comments  
\newcommand{\kibitz}[2]{\ifnum\Comments=1\textcolor{#1}{#2}\fi}

\usepackage[scaled=.9]{helvet}

\usepackage{caption}
\usepackage[table,dvipsnames]{xcolor}
\usepackage[colorlinks=true, linkcolor=blue, citecolor=blue]{hyperref}
\usepackage{style}
\usepackage{wrapfig}

\renewcommand{\bibname}{References}
\renewcommand{\bibsection}{\subsubsection*{\bibname}}
\bibpunct{(}{)}{;}{a}{,}{,}

\renewcommand{\hat}{\widehat}
\renewcommand{\tilde}{\widetilde}

\title{Split Localized Conformal Prediction}
\author{
	Xing Han \qquad Ziyang Tang \qquad Joydeep Ghosh \qquad Qiang Liu \\
    University of Texas at Austin\\
	\texttt{\{aaronhan223, ztang, jghosh, lqiang\}@utexas.edu}
}

\date{}
\begin{document}




\maketitle

\begin{abstract}
Conformal prediction is a simple and powerful tool that can quantify uncertainty without any distributional assumptions. Many existing methods only address the average coverage guarantee, which is not ideal compared to the stronger conditional coverage guarantee. Existing methods of approximating conditional coverage require additional models or time effort, which makes them not easy to scale. In this paper, we propose a modified non-conformity score by leveraging the local approximation of the conditional distribution using kernel density estimation. The modified score inherits the spirit of split conformal methods, which is simple and efficient and can scale to high dimensional settings. We also proposed a unified framework that brings together our method and several state-of-the-art. We perform extensive empirical evaluations: results measured by both average and conditional coverage confirm the advantage of our method.
\end{abstract}
\setcitestyle{numbers,square}


\section{Introduction}
Quantifying uncertainty is an important problem for reliable decision-making in many applications. 
Conformal prediction \citep{lei2018distribution, shafer2008tutorial, angelopoulos2021gentle} is a simple and powerful framework 
that can provide a valid coverage in finite samples without any distributional assumption. Recent works based on existing methods of conformal prediction \citep[e.g.][]{papadopoulos2002inductive,vovk2005algorithmic,lei2015conformal,romano2019conformalized,tibshirani2019conformal}
bring distribution-free uncertainty quantification to a range of applications,
including change-points detection \citep{volkhonskiy2017inductive},
counterfactuals and individual treatment \citep{lei2020conformal},
survival analysis \citep{candes2021conformalized, teng2021t},
outlier detection \citep{bates2021testing},
confidence set of image classification \citep{makili2012active, matiz2019inductive, angelopoulos2020uncertainty, stutz2022learning}, time series \citep{zaffran2022adaptive, lin2022conformal, gibbs2021adaptive}
and meta-learning \citep{fisch2021few}.



\begin{figure*}[t]
    \begin{minipage}{\textwidth}
    \centering
    \begin{tabular}{@{\hspace{-2.6ex}} c @{\hspace{-2.6ex}} c @{\hspace{-2.6ex}}}
    \begin{tabular}{c}
    \vspace{-0.7em}
         \includegraphics[width=0.27\textwidth]{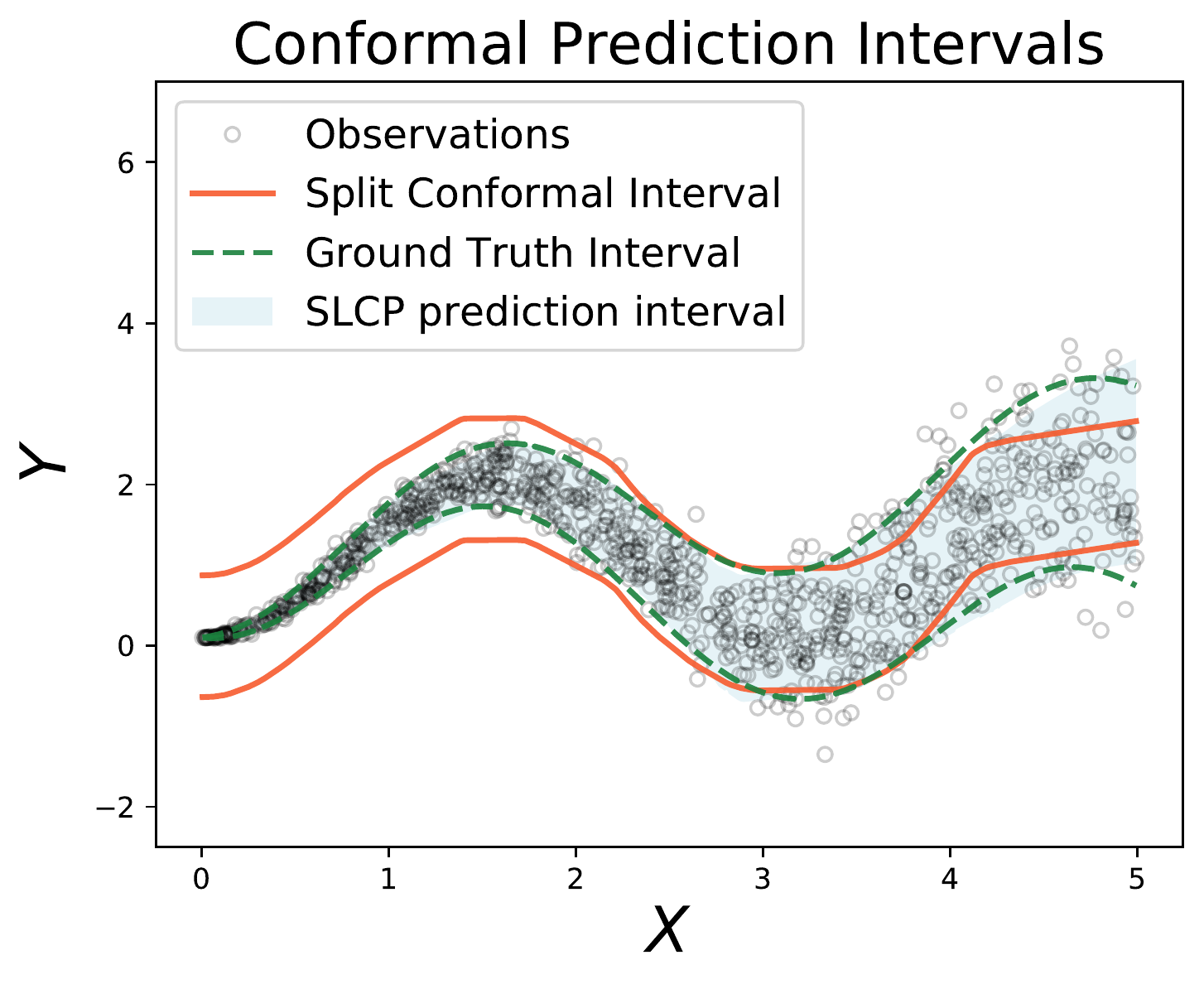} 
         \\
         {\small{(a)}}
         \end{tabular} &
     \begin{tabular}{c}
        \includegraphics[width=0.72\textwidth]{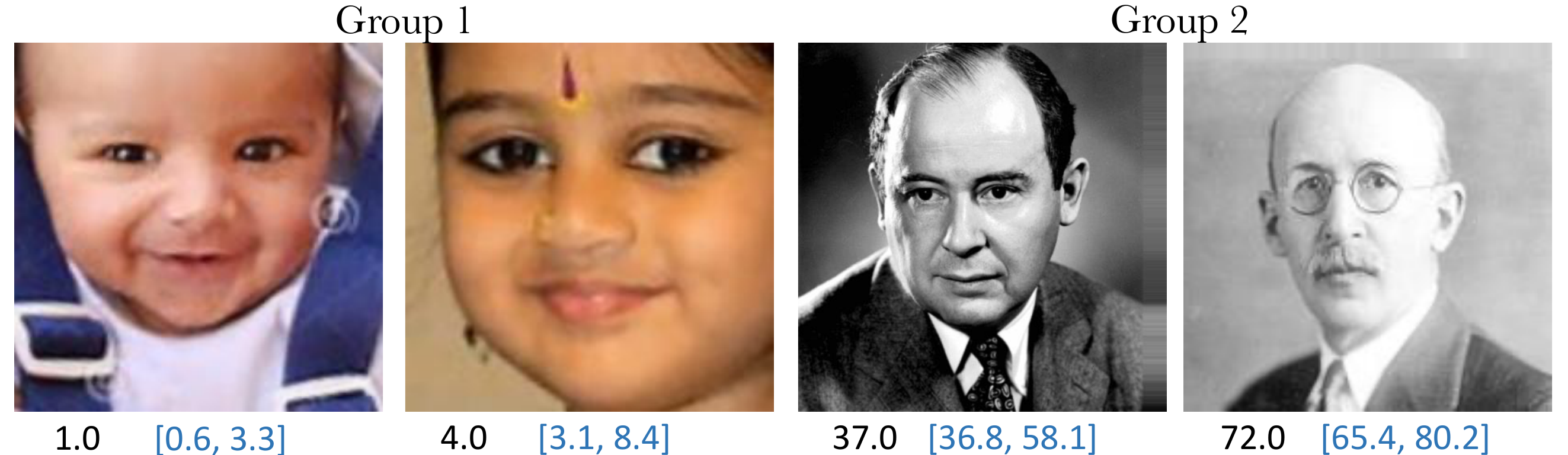} 
        \\
        {\small{(b)}}
        \end{tabular} \\
    \end{tabular}
    \end{minipage}
    \caption{Examples of the proposed conformal PI in different data modalities: (a) constructing adaptive intervals on data with heterogeneous noise; (b) age prediction on two distinct subgroups from the IMDB-WIKI dataset \citep{Rothe-IJCV-2018}, black numbers are ground truth labels and blue numbers are generated PI.}
    \label{fig:age_regression}
\end{figure*}

Although conformal prediction provides distribution-free coverage guarantee,
the coverage it provides is an average guarantee over the marginal distribution of features.
This guarantee relies on the marginal distribution of features, which may fail in covariate shift and cannot give any guarantee on observed sample from the new distribution.
For instance, when we train a model to construct predictive interval (PI) on a population under groups of different race with significantly different proportion, simply construct PI for the whole population is not enough, since it is dominated by the majority group, whereas the constructed PI for the minority group cannot be certified.
Therefore, it is necessary to build up a stronger coverage guarantee that incorporates conditions. However, conditional coverage guarantee has been proven to be theoretically impossible \citep{lei2014distribution, barber2019limits} under arbitrary distribution assumption. One can still try to approximate the conditional coverage under certain assumptions. 
Existing works \citep[e.g.][]{romano2019conformalized, guan2019conformal} leverage different approaches to approximate the conditional coverage guarantee. However, their methods either require additional quantile models in different mis-coverage rate \citep{romano2019conformalized}, or need to inversely grid search for the best parameter which is time consuming~\citep{guan2019conformal}. We aim to develop a method that is simple and easy to use without performing too much modification from the split conformal framework~\citep{papadopoulos2002inductive, lei2015conformal}.

Our method is motivated by an ideal scenario, where if the non-conformity score and the features are mutually independent, the conditional coverage guarantee is automatically satisfied since the conditional distribution of the score is identical to the marginal distribution.
We propose an approach to approximate the conditional distribution by leveraging local information in the training set,
and use it to construct a new non-conformity score which is later plugged in the same procedure similar to the split conformal methods.
The new method, called \textit{Split Localized Conformal Prediction} (SLCP), is built upon the split conformal framework. Therefore, it provides the same average coverage guarantee as split conformal methods.
In addition, by leveraging local information, it approximates the conditional guarantee and provides a more precise PI. Building upon SLCP, we propose a more general conformal prediction framework that unifies relevant state-of-the-art with potential scaling to new methods.
Empirical results demonstrate the strength of our method, which provides adaptive PI with sufficient conditional coverage. In particular, SLCP can scale to high dimensional tasks such as age prediction with respect to different sub-populations, an example of this application is illustrated in Figure~\ref{fig:age_regression}. We hope our work spurs an interest in the community in developing better localized conformal prediction methods for different applications.



\section{Problem Settings}


Given an {\it i.i.d.} regression dataset $(X_i, Y_i)_{i=1}^n \in (\R^d \times \R)^n$ consisting of a $d$-dimensional feature vector $X_i$ and the response variable $Y_i$,
we aim to provide a confidence band $\hat{\C}(x) \subseteq \R$ from the existing data that contains the unknown response variable $Y$ with high probability.
Specifically, $\hat{\C}(x)$ should satisfy the following
\textit{average coverage} guarantee:
\begin{equation}\label{eqn:average_coverage}
    \P(Y_{n+1} \in \hat{\C}(X_{n+1})) \geq 1 - \alpha,
\end{equation}
where $\alpha \in (0, 1)$ is a prespecified miscoverage level, and $(X_{n+1}, Y_{n+1})$ is a new data point which is {\it i.i.d.} drawn from the same distribution. 


\myparagraph{Quantile Lemma}
The idea of conformal prediction is based on the quantile lemma \citep{tibshirani2019conformal}, 
which states that 
for any {\it i.i.d.} (or exchangeable) sequence of random variables 
 $\{V_i\}_{i=1}^n$, $V_{n+1}$ lies in the 
empirical $(1-\alpha)$-quantile value of the previously seen training data with probability at least $1-\alpha$:
\begin{equation}\label{eqn:quantile_lemma}
    \P(V_{n+1}\leq \Q(1-\alpha, \{V_i\}_{i=1}^n))\geq 1-\alpha,
\end{equation}
where $\Q(1-\alpha, \{V_i\}_{i=1}^n)$ denotes the 
$\lceil (1-\alpha) (n+1)\rceil$ smallest value in $\{V_i\}_{i=1}^n$.


\myparagraph{Split Conformal Prediction.}

Split conformal prediction \citep{papadopoulos2002inductive, lei2015conformal}, also known as inductive conformal prediction, is a simple way to conduct conformalized prediction bands.
Split conformal prediction consists of three steps:
(1) split the training data into two disjoint subsets, 
a proper training set $\I_{train}$ and a calibration set $\I_{cal}$;
(2) fit $\I_{train}$ to a regression model to obtain the mean model $\hat{\mu}$;
(3) compute the non-conformity score for each calibration sample as
\begin{equation}\label{eqn:absolute_nc}
    V_i = |Y_i - \hat{\mu}(X_i)|
\end{equation}
to measure how bad the points $(X_i, Y_i)$ ``conform'' to the model. 
By quantile lemma in Eq \eqref{eqn:quantile_lemma} we have:
\begin{align}\label{eqn:split_conformal}
    \P(V_{n+1}\leq &d) \geq 1-\alpha \nonumber \\ 
    &\Rightarrow \P(Y_{n+1}\in \hat{\mu}(X_{n+1}) \pm d)\geq 1-\alpha,
\end{align}
where $d = \Q(1-\alpha, \{V_i\}_{i\in \I_{cal}})$ is the empirical quantile value in the calibration set.

\myparagraph{Asymmetric Conformal Band.}
In Eq. \eqref{eqn:split_conformal}, the conformal band $\hat{\C}(x) = [\hat{\mu}(x) - d, \hat{\mu}(x) + d]$ is a symmetric band with constant length. 
The reason a symmetric band is obtained is the absolute value in Eq. \eqref{eqn:absolute_nc}. \citet{linusson2014signed} modified $V_i$ to achieve an asymmetric band.
Denote that
\begin{equation}
    V_{i,1} = Y_i - \hat{\mu}(X_i),\quad V_{i,2} = \hat{\mu}(X_i) - Y_i, \nonumber
\end{equation}
similar to Eq. \eqref{eqn:split_conformal}, let $\Q_{\alpha_1} = \Q(1-\alpha_1, \{V_{i,1}\}_{i=1}^n)$ and $\Q_{\alpha_2} = \Q(1-\alpha_2, \{V_{i,2}\}_{i=1}^n)$. Based on the quantile lemma, we can derive the following results:
\begin{align}
    \P(Y_{n+1} &\leq \hat{\mu}(X_i) + \Q_{\alpha_1}) \geq 1-\alpha_1, \nonumber \\ 
    \P(Y_{n+1} &\geq \hat{\mu}(X_i) - \Q_{\alpha_2}) \geq 1-\alpha_2.
\end{align}
By union bound, let 
\begin{equation}
    \hat{\C}(x) = [\hat{\mu}(X_i) - \Q_{\alpha_2}, \hat{\mu}(X_i) + \Q_{\alpha_1}] \nonumber
\end{equation}
and set $\alpha = \alpha_1 + \alpha_2$, we then have:
\begin{align*}
    \P(Y_{n+1} \in \hat{\C}(X_{n+1}))\geq 1-\alpha.
\end{align*}



\myparagraph{Conditional Coverage Guarantee.}
Conformal prediction provides \textit{average coverage} guarantee in Eq. \eqref{eqn:average_coverage}, which needs to be marginalized among $X_{n+1}$ in the same distribution.
It is not guaranteed that for any $X=x$, the conformal band $C(x)$ can cover the response of $Y|X=s$, or as the notion of conditional coverage guarantee as follows:
\begin{equation}\label{eqn:conditional_coverage}
    \P(Y_{n+1}\in \hat{\C}(x)|X_{n+1}=x) \geq 1-\alpha.
\end{equation}


The average coverage guarantee by split conformal prediction is very limited.
For example, Figure \ref{fig:age_regression} (a) demonstrates an example of data with heteroscedastic noise, where the noise distribution is changing w.r.t. $x$. In this case, a constant interval length generated by marginal conformal coverage in Eq. (\ref{eqn:split_conformal}) cannot provide the optimal confidence interval.

\citet{vovk2012conditional, lei2014distribution} show the hardness of achieving such conditional guarantee, where any method satisfies the property in Eq. \eqref{eqn:conditional_coverage} must yield an interval $\hat{\C}(x)$ with infinite expected length at any non-atom point $x$.

\section{Split Localized Conformal Prediction} \label{sec:slcp}

Our goal is to provide a confidence band $\hat{\C}(x)$ for the incoming sample $(X_{n+1}, Y_{n+1})$ 
that satisfies the average coverage guarantee in Eq.~\eqref{eqn:average_coverage}.
Meanwhile, we also want $\hat{\C}(x)$ to approximate the conditional coverage in Eq.~\eqref{eqn:conditional_coverage}.


\subsection{Ideal Scenario v.s. Real World Scenario} \label{sec:scenarios}

Figure~\ref{fig:age_regression} (a) illustrates that split conformal prediction generally does not provide any conditional coverage guarantee.
However, to motivate the new algorithm,
we may start by looking into the ideal scenario
when split conformal prediction can work perfectly well.
\citet{lei2018distribution} states that when the true underlying distribution of $Y$ satisfies
$$
    Y = \mu(X) + \varepsilon, \quad \mathrm{where} ~\E[\varepsilon|X]=0,
$$
if $\varepsilon$ is independent of $X$, 
and our estimated conditional mean model converges to the true $\mu$,
then split conformal prediction provides oracle confidence bands that satisfy the conditional coverage guarantee.

The reason split conformal prediction works well in this scenario is
the non-conformity score $V(X,Y) = Y-\hat{\mu}(X)$ is approximating the feature independent noise $\varepsilon$.
The conditional distribution of $V$ given $X=x$ is agnostic to the choice of $x$.
Specifically, if $\hat{\mu} = \mu$,
we will have:
\begin{align*}
    \P(V_{n+1}\leq d|X=x) = \P(V_{n+1}\leq d) \geq 1-\alpha,\,
\end{align*}
where $d$ is chosen independently to $V_{n+1}$ to satisfy the average coverage guarantee, i.e. the empirical quantile of $\{V_i\}$ in the calibration set.
The primary reason is the conditional distribution of $V|X=x$
is identical to the marginal distribution of $V$.

However, in real world scenario this is unlikely to happen.
First, estimation of $\mu$ cannot be accurate due to model misspecification.
Second, the noise $\varepsilon$ can be heteroscedastic, meaning that $\varepsilon = \varepsilon(X)$ is correlated to $X$.
Both reasons result in correlation between the non-conformity score $V$ and $X$,
which prohibits the conditional distribution $V|X=x$ from being identical to the marginal distribution $V$.

\subsection{Proposed Method}
\label{sec:proposed_method}

\myparagraph{Motivation.}
The inconsistency between conditional distribution $V|X=x$ and marginal distribution $V$ is an obstacle for split conformal prediction to achieve conditional coverage guarantee. It achieves average coverage guarantee by choosing $d$ to be empirical quantile over the marginal distribution $V$, rather than the conditional distribution $V|X=x$. Assume many samples are drawn from $V|X=x$,
if we choose $d$ to be the empirical quantile of those samples, we satisfy the conditional guarantee for $V|X=x$.
This can happen when the domain of $X$ is finite and discrete, and each $X$ occurs with sufficient times.
However, in continuous domain, each value of the sample $X$ can be drawn only once,
and their corresponding non-conformity score $V$ is the only sample from the conditional distribution $V|X=x$.

Although the real conditional distribution $V|X=x$ cannot be obtained, one can approximate $V|X=x$ using local information for every $x$.
For example, for the type of $V_i$ with $X_i$ close to $x$, we can assign a high weight on sample $X_i$; and for the type of $V_i$ with $X_i$ far away from $x$, we can assign a lower weight on the sample.

\myparagraph{Localized Approximation.}
Formally, we consider $\hat{F}_h(v|X=x)$ that approximate the conditional CDF function of $F(v|X=x) = \P(V\leq v|X=x)$ by localization with kernel smoothing, also known as NW estimator~\citep{nadaraya1964estimating, watson1964smooth}
\begin{align}\label{eqn:kde_cdf}
    &\hat{F}_h(V=v|X=x) = \sum_{i\in \I_{train}} w_h(X_i|x) \mathds{1}_{v\leq V(X_i, Y_i)},  \\ 
    &\textrm{where}~~w_h(X_i|x) = \frac{K(\|\vv{f}(X_i)-\vv{f}(x)\|/h)}{\sum_{j\in \I_{train}} K(\|\vv{f}(X_j)-\vv{f}(x)\|/h)}.\notag\,
\end{align}
Here, $K(\cdot)$ is any kernel function, and $\vv{f}$ is an embedding function. 
In low dimensional space, $\vv{f}$ can be set as an identity function;
in high dimensional space, such as training a neural network regressor, we usually adopt the second last layer of the network as our embedding function. We can now use the plug-in estimator $\Q(\alpha, \hat{F}_h(V|X=x))$ to estimate the conditional quantile.
Our hope is that the empirical quantile of the localized approximation distribution is close to the quantile of the true conditional distribution:
\begin{align}
    &\P(V \leq \Q(\alpha, \hat{F}_h(V|X=x)) | X=x) \nonumber \\
    &\approx \P(V \leq \Q(\alpha, F(V|X=x)) | X=x) = \alpha,~\forall x.
\end{align}
This approximation yields a direct way of constructing confidence band $\hat{\C}(x) = \{y: V(x,y) \leq \Q(\alpha, \hat{F}_h(V|X=x))\}$ and with the hope that $\P(y\in \hat{\C}(x))\approx \alpha$.
However, since $\hat{F}_h$ is just an approximation of the ground truth, 
we don't have any coverage guarantee on $\P(V_{n+1}\leq \Q(\alpha, \hat{F}_h(V|X_{n+1})))$. The proposed method, SLCP, is to view
\begin{equation}\label{eqn:decorrelated_score}
    V^{\alpha,h}(x,y) = V(x,y) - \Q(\alpha, \hat{F}_h(V|X=x))
\end{equation}
as a new non-conformity score, 
and apply conformal correction over the new score $V_{i}^{\alpha,h} = V_{i}^{\alpha,h}(X_i,Y_i)$ in the calibration set. 
The new confidence band can be written as:
\begin{equation}\label{eqn:SLCP_band}
    \hat{\C}^{\rm SLCP}(x) = \{y: V^{\alpha,h}(x,y)\leq \Q(\alpha, \{V^{\alpha,h}_i\}_{i\in \I_{cal}})\},
\end{equation}
where the threshold is set to be the empirical quantile value of the new non-conformity score. We summarize the essential steps of SLCP in Algorithm~\ref{alg:LDCP}
\begin{pro}\label{pro:average_guarantee_SLCP}
The PI in Eq.~\eqref{eqn:SLCP_band} provides average coverage guarantee:
\begin{equation}
    \P(Y_{n+1}\in \hat{\C}^{\rm SLCP}(X_{n+1})) \geq \alpha.
\end{equation}
\end{pro}

\begin{pro}\label{pro:conditional_coverage}
As $n\rightarrow \infty$, the PI in Eq. \eqref{eqn:SLCP_band} achieves asymptotic conditional coverage:
\begin{equation}
    \P(Y_{n+1}\in \hat{\C}^{\rm SLCP}(X_{n+1})|X_{n+1}) \geq \alpha.
\end{equation}
\end{pro}

\begin{remark}
The approximation of conditional CDF of $V$ in Eq~\eqref{eqn:kde_cdf} is the weighted sum of the check function over the training set: 
since SLCP will compute the empirical quantile later on in the calibration set to guarantee average coverage.
If we perform approximation on the calibration set, the later step will have dependency over the previous step, which violates the spirit of split conformal methods and has to reduce back to the original full conformal prediction \citep[e.g.][]{guan2019conformal, tibshirani2019conformal}.
\end{remark}

\subsection{Unify Localized Conformal Prediction Methods}
\label{sec:generalized_framework}
The new non-conformity score in Eq.~\eqref{eqn:decorrelated_score} can be extended to any monotonic function that leverages the approximation of conditional CDF. We now propose a generalized framework that connects SLCP with other state-of-the-art methods.
\begin{ass}[Monotonic Score Function]\label{ass:monotonic}
    We call $m$ a monotonic score function if $m$ is a real-value function such that
    for any real value $v\in \R$ and a CDF function $F$, if $\theta_1 \geq \theta_2$, we have:
    $$
    m(v,F,\theta_1) \geq m(v, F, \theta_2).
    $$
    We also denote the (semi-)inverse function of $m$ as $p$ as:
    $$
    p(v,F, t) = \inf_{\theta\in \R}\{\theta: m(v,F,\theta) \geq t\}.
    $$
\end{ass}

For simplification, let $\hat{F}_i = \hat{F}_h(V|X=X_i)$. With any monotonic function $m$, we present the following theorem.
\begin{thm}[Generalized SLCP]\label{thm:generalized_SLCP}
    Suppose $m(v, F, \cdot)$ is a score function defined in Assumption~\ref{ass:monotonic} with its inverse function $p(v,F,\cdot)$. 
    Let $\theta_i = p(V_i, \hat{F}_i, 0)$, and suppose $\theta^* = \Q(\alpha, \{\theta_i\}\cup \{\infty\})$
    we have:
    \begin{equation}
        \Pr\{m(V_{n+1}, \hat{F}_{n+1}, \theta^*) \leq 0\} \geq \alpha.
    \end{equation}
    Moreover, let $\hat{\C}(x) = \{y: m(v, \hat{F}, \theta^*)\leq 0\}$, where $v$ is the non-conformity score for $V(x,y)$, and $\hat{F}$ is the approximate CDF conditioned on $x$, we have:
    \begin{equation}
        \Pr\{Y_{n+1} \in \hat{\C}(X_{n+1})\} \geq \alpha.
    \end{equation}
\end{thm}

\begin{remark}
Proposition~\ref{pro:average_guarantee_SLCP} can be viewed as a special case of Theorem~\ref{thm:generalized_SLCP}, where the non-conformity score in Eq. \eqref{eqn:decorrelated_score} is a special case of the monotonic score function such that
$
    m(v,F,\theta) = \Q(\alpha, F) + \theta - v.
$
If we set 
$m(v,F,\theta) = \Q(\theta, F) - v$, 
we obtain a split conformal version of \citet{guan2019conformal}.
Moreover, if we do not leverage localized information, by making $m(v, F,\theta)$ only related to the original non-conformity score, 
for example, $m(v,\theta) = \theta - v$, we can then obtain the vanilla split conformal prediction.
We provide proofs of Proposition \ref{pro:conditional_coverage}, Theorem \ref{thm:generalized_SLCP}, and discussions of the monotonic score function in the Appendix.
\end{remark}

\begin{algorithm*}[t] 
    \caption{Split Localized Conformal Prediction (SLCP)}  
    \label{alg:LDCP}
    \begin{algorithmic} 
    \STATE {\bf Input}: 
    Training Data $(X_i, Y_i)_{i=1}^n$;
    Miscoverage level $\alpha = \alpha_1 + \alpha_2$;
    Kernel function $K(\cdot)$;
    Bandwidth $h$.
    \STATE {\bf Process}:\\
    \quad Randomly split dataset index $\{1,\ldots,n\}$ as $I_{train}$ and $I_{cal}$;\\
    \quad Fit model with $I_{train}$ with training algorithm and obtain estimation: $\hat{\mu}$; \\
    \quad (Optional) Use the second last layer output of the Neural Network as the embedding function $\vv{f}$.\\
    \quad Set asymmetric nonconformity score $V_{i,1} = Y_i - \hat{\mu}(X_i)$ and $V_{i,2} = \hat{\mu}(X_i) - Y_i,~\forall i\in \I_{train}\cup\I_{cal}$.\\
    \quad Compute $V_{i,1}^{\alpha_1,h}$ and $V_{i,2}^{\alpha_2,h}$ for all $i\in \I_{cal}$ based on Eq.~\eqref{eqn:decorrelated_score} as:
    $$
    V_{i,1}^{\alpha_1,h} = V_{i,1} - \Q(\alpha_1, \hat{F}_h(V_1|X=X_i)),\quad 
    V_{i,2}^{\alpha_2,h} = V_{i,2} - \Q(\alpha_2, \hat{F}_h(V_2|X=X_i)),~\forall i \in \I_{cal}.
    $$
    \STATE {\bf Output}:
    Conformal band 
    \begin{align}\label{eqn:localized_band}
        \hat{\C}^{\rm SLCP}(x) = [\hat{\mu}(x)-\underbrace{\textcolor{cyan}{\Q(\alpha_2, \hat{F}_h(V_2|X=x))}}_{\rm{ Empirical~Quantile}}-
    \underbrace{\textcolor{violet}{\Q(\alpha_2, \{V_{i,2}^{\alpha_2,h}\}_{i\in \I_{cal}})}}_{\rm{Conformal~Correction}}&, \notag\\
    \hat{\mu}(x)+\textcolor{cyan}{\Q(\alpha_1, \hat{F}_h(V_1|X=x))}
    +\textcolor{violet}{\Q(\alpha_1, \{V_{i,1}^{\alpha_1,h}\}_{i\in \I_{cal}})}&].
    \end{align}
    \end{algorithmic} 
\end{algorithm*}



\subsection{Practical Implementation}
\label{sec:practical}
\paragraph{Choice of Non-conformity Score $V$}
The simplest choice of $V$ is by setting $V=Y$ with a naive constant model.
Surprisingly, even without any regression model, our method can achieve a good performance
in terms of coverage rate and interval length.
The reason $Y$ can serve as a good non-conformity score in SLCP is $Q(\alpha, \hat{F}(Y|X=x))$ itself can be viewed as a non-parametric regression for conditional quantile value of $Y|X=x$.
In practice, we choose $V = \pm(Y - \hat{\mu}(X))$ and use SLCP to construct the new non-conformity score.
Here, $\hat{\mu}$ can be the estimation of either conditional mean or quantile model using the training set data.
The benefit of this non-conformity score is that the inverse computation of the final confidence band $\hat{\C}(x)$
is just inverting the conformal correction operation, where the inversion is simply a plus or minus operation and results in an easy-to-compute confidence band.

\myparagraph{Asymmetric Conformal Prediction}
We adopt an asymmetric nonconformity score to construct the prediction band.
As shown in Algorithm \ref{alg:LDCP}, this gives us a flexible framework that can later be extended to conformal prediction methods with multiple regressors such as conformalized quantile regression. Additionally, the asymmetric nonconformity score will also lead to a straightforward theoretical analysis.

\myparagraph{Choice of Kernel and Bandwidth}
We mainly employed the Gaussian RBF kernel $K(x) = \frac{1}{\sqrt{2\pi}}e^{-\frac{x^2}{2}}$ in our experiments. As an ablation study, 
we also compared with different choices of kernels such as Boxcar $K(x) = \frac{1}{2}I(x)$ and Epanechnikov $K(x) = \frac{3}{4}(1 - x^2)I(x)$, where $I(x) = 1$ if $|x| \leq 1$ and 0 otherwise. \citet{nadaraya1964estimating} and \citet{watson1964smooth} suggest a bias-variance trade-off on the selection of bandwidth, where we performed ablation studies in the Appendix. In our experiments, we simplify this procedure by using the ``median trick'', where $h$ is the median of the pairwise distance of $\|\vv{f}(X_i)-\vv{f}(X_j)\|$ for all different $i\neq j$ in the training set. This choice of bandwidth achieves a competitive empirical performance.

\myparagraph{Accelerate NW estimation using Mini-Batch}
With a large amount of training/calibration data, the computation of the NW estimation in Eq~\eqref{eqn:kde_cdf} becomes a bottleneck, which requires at least $\mathcal{O}(|\I_{train}|\cdot |\I_{cal}|)$ time complexity to obtain the pairwise distance between the training and calibration set.
In order to accelerate the computation,
we introduce a slightly modified version of SLCP, which greatly reduce the computational time without hampering the performance too much.
The new approach samples a mini-batch of training data $\mathcal{B}\subseteq \mathcal{I}_{train}$ and compute the NW estimation in Eq~\eqref{eqn:kde_cdf} with $\mathcal{B}$ replacing $\mathcal{I}_{train}$. The computation complexity reduces to $\mathcal{O}(|B|\cdot |I_{cal}|)$ and if we use a fixed size of the batch, i.e. $|B| = k$, we can therefore obtain a linear computation complexity to the sample size. 


\section{Related Works}
\myparagraph{Conformal Prediction Methods with Localization}
\citet{guan2019conformal,guan2021localized} proposed a localized conformal prediction method to construct prediction intervals.
They are also inspired by the localized conditional distribution in Eq.~\eqref{eqn:kde_cdf}.
As mentioned, both \citet{guan2019conformal}'s method and ours can be unified in our generalized framework in Section~\ref{sec:generalized_framework} under different score functions. However, the major difference is that they compute the empirical quantile $\hat{F}$ under the calibration dataset, which requires a full conformal style of inverse computation for $y$ while also performing a grid search for the best mis-coverage rate. The full conformal style of computation restricts the scalability of this method, which leads to a much slower computation compared with SLCP. 
A more recent work~\citep{lin2021locally} follows \citet{guan2019conformal}'s idea, where they further optimize the kernel regression parameter to obtain the weight in Eq.~\eqref{eqn:kde_cdf}.
Different from our method, they directly set the threshold in Eq.~\eqref{eqn:SLCP_band} as 0 without the final conformal correction.
Moreover, they do not consider the asymmetric conformal band and their results tend to be over-conservative. In addition, earlier work from \citet{lei2018distribution} proposed to use scaled residuals to locally adapt to heteroskedasticity. However, the extra variability from estimating the denominator could lead to inflated PIs.



The other line of work by \citet{izbicki2019flexible} proposed Dist-Split and CD-Split. In particular, Dist-Split belongs to our generalized framework if we set $V=Y$ and apply the same monotonic score function $m$ as \citet{guan2019conformal}. CD-Split is built upon the locally valid prediction band from \citet{lei2014distribution}, where its Euclidean distance used to partition the feature space is replaced by a ``profile distance'': it measures the distance of conditional density functions. HPD-Split \citep{izbicki2020cd} is a variant of CD-Split that requires less tuning. This line of methods uses the conditional density estimator proposed by \citet{izbicki2017converting} instead of the NW estimator.

\myparagraph{Conformalized Quantile Regression}
\citet{romano2019conformalized} proposed CQR to obtain a better non-conformity score based on the estimated conditional quantile functions. However, it is generally more difficult to train an accurate quantile function in high-dimensional settings. Moreover, as mentioned by \citet{lin2021locally}, CQR is not a post-hoc method and needs to retrain the quantile models every time when more than one coverage level is desired.
In comparison, SLCP can construct conformal band Eq. \eqref{eqn:localized_band} using either mean regression model $\hat{\mu}$ or quantile regression models. Therefore, SLCP does not need to train a new model with different coverage levels $\alpha$. Similar work by \citet{feldman2021improving} improves the conditional coverage of standard quantile regression by modifying the loss function. But this method does not provide a finite sample coverage guarantee and does not belong to the category of conformal prediction.
\myparagraph{Other Methods to Approximate Conditional Coverage}
Restricted approximated conditional coverage \citep{barber2019limits} is another method that separates the whole domain into a set of sub-domains $\{\X_i\}$ 
with small probability mass $\P(\X_i)\leq \delta$,
and relax the conditional coverage into marginalized coverage in each sub-domain: $\P(Y\in \C(X)|X\in \X_i)\geq 1-\alpha,~\forall \X_i.$
However, the performance depends heavily on the predefined separation, and prediction intervals near the boundary of each sub-domain are inferior. \citet{tibshirani2019conformal} discussed conformal prediction under covariate shift and proposed an initial version of approximating conditional coverage guarantee based on localized information. However, they only consider localization around a fixed point $x_0$. In addition, \citet{sesia2021conformal} proposed a histogram method to estimate the conditional distribution and achieved asymptotic conditional coverage guarantees for black-box models. \citet{chernozhukov2021distributional} constructed robust conditional PI based on quantile and distribution regression. \citet{kivaranovic2020adaptive} designed average and conditional PI for deep neural networks via optimizing a novel loss function. Full conformal prediction and Jackknife methods \cite{kim2020predictive} have also been studied to provide a predictive coverage guarantee. But these methods are less efficient and over-conservative in prediction intervals \citep{lei2018distribution}. Different from the conformal regression problem, conformal classification \citep{makili2012active, matiz2019inductive} aims to build a ``prediction set'' where the total probability covered by the set should reach the specified level. There are several recent works approach this problem in terms of adaptivity \citep[e.g][]{angelopoulos2020uncertainty, einbinder2022training, romano2020classification} or efficiency \citep{stutz2022learning} and achieved state-of-the-art results.

\section{Experiments} \label{sec:exp}
\begin{figure*}[t!]
    \begin{minipage}{\textwidth}
    \centering
    \begin{tabular}{@{\hspace{-2.6ex}}c @{\hspace{-2.6ex}} c @{\hspace{-2.6ex}} c}
        \begin{tabular}{c}
        \includegraphics[width=.32\textwidth]{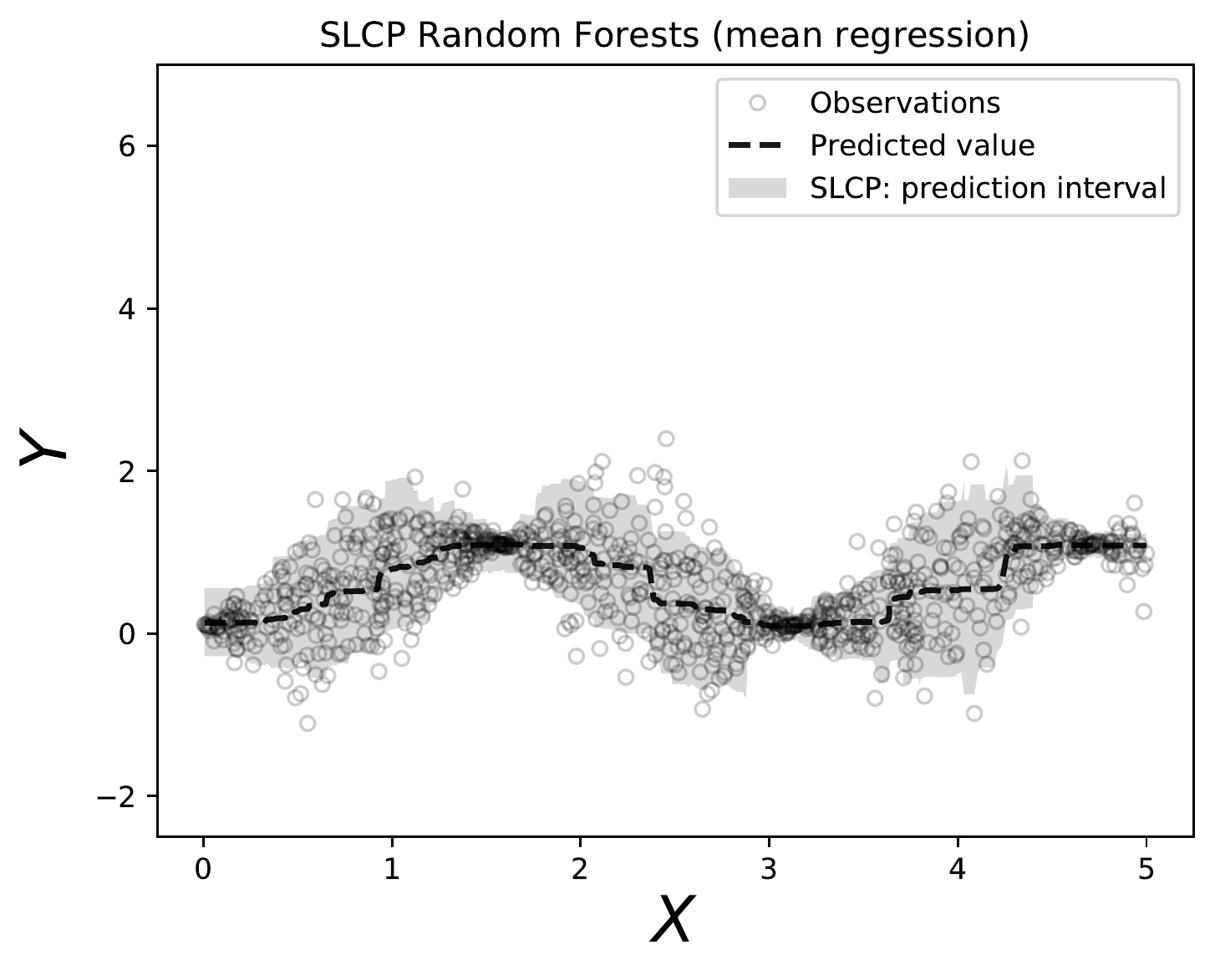} 
        \\
        {\small{(a) Avg. cov. 90.36\%; Avg. len. 1.27}}
        \end{tabular} &
        \begin{tabular}{c}
        \includegraphics[width=.32\textwidth]{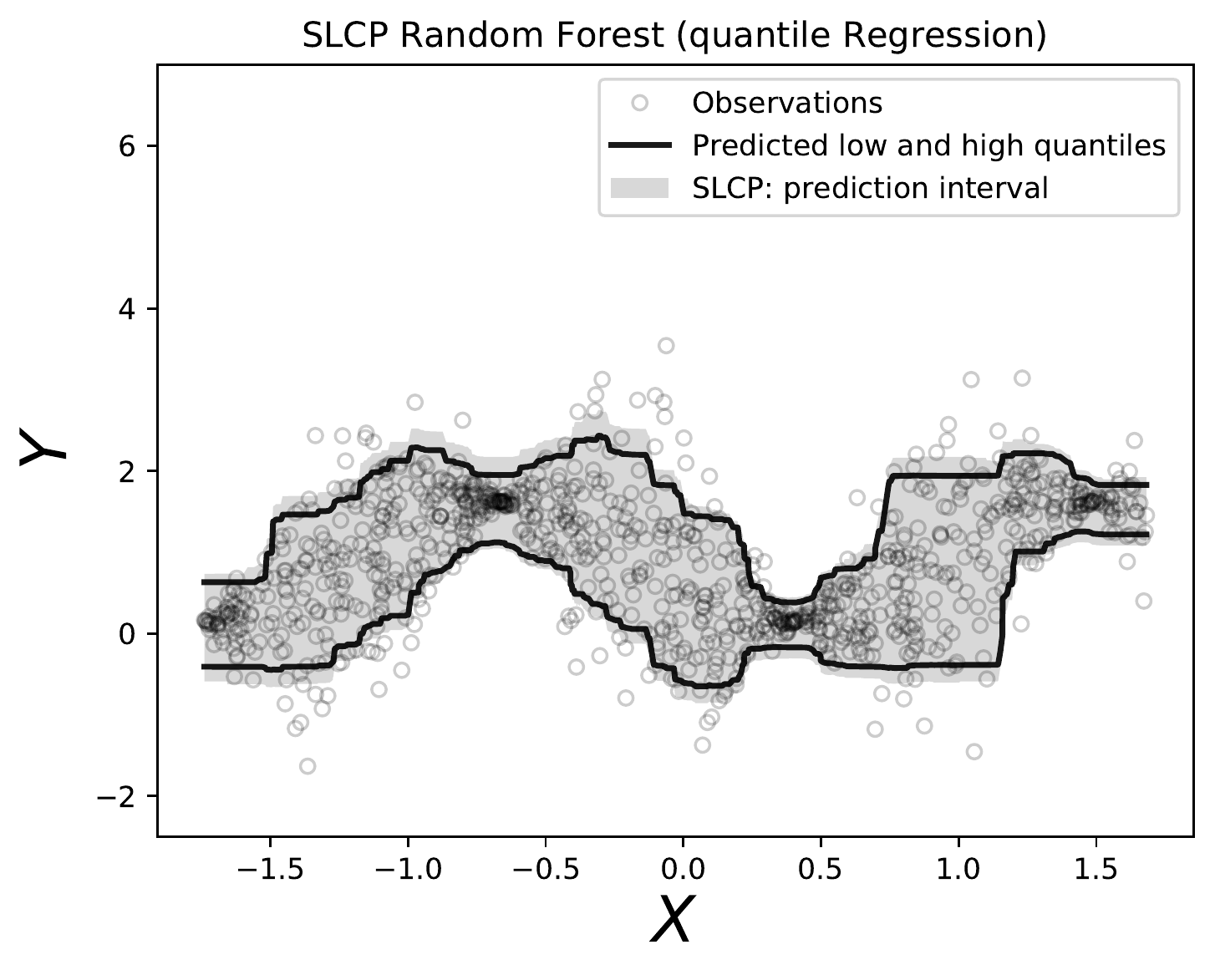}
        \\
        {\small{(b) Avg. cov. 90.06\%; Avg. len. 1.86}}
        \end{tabular} &
        \begin{tabular}{c}
        \includegraphics[width=.32\textwidth]{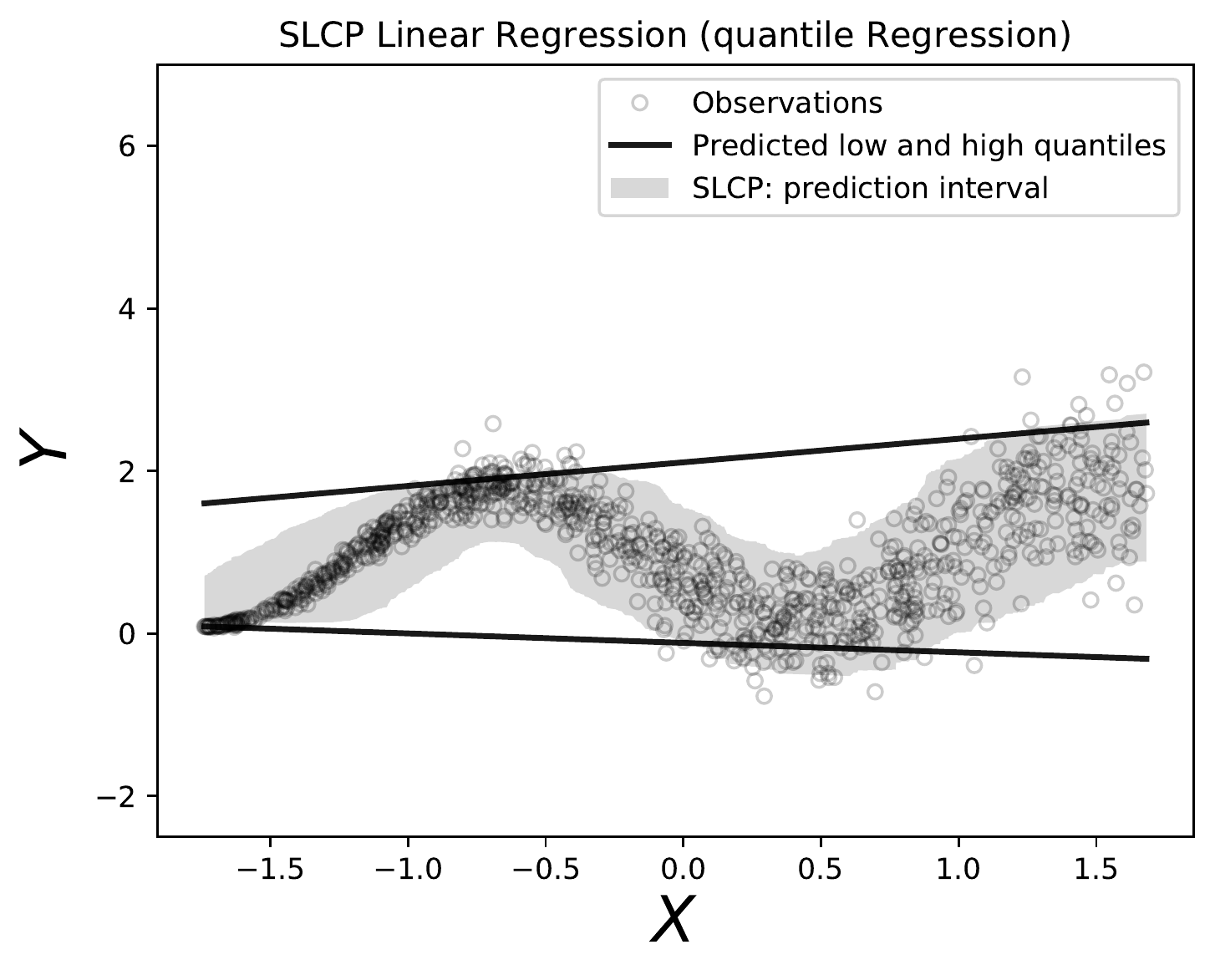} 
        \\
        {\small{(c) Avg. cov. 90.52\%; Avg. len. 1.52}}
        \end{tabular} \\
        \end{tabular}
    \end{minipage}
    \vspace{-1.5ex}
    \caption{
    Prediction intervals constructed by SLCP on multiple synthetic datasets. (a) SLCP with a mean random forest estimator; (b) SLCP with quantile random forest estimators; (c) SLCP with quantile linear regression. Results have shown that SLCP provides adaptive intervals with coverage guarantees and is robust to misspecified models.}
    \label{fig:all_sim}
\end{figure*}

In this section, we empirically evaluate the proposed method, SLCP, over other state-of-the-art solutions on various datasets and experiments. Overall, we show that SLCP is an effective solution to locally approximate conditional coverage in different situations. The constructed PI is robust even if the models are misspecified to the inherent nature of given data. We start by discussing the experiment settings in section \ref{sec:exp1}, followed by the average and conditional coverage results on simulated and real datasets in section \ref{sec:exp2}, \ref{sec:exp3} and \ref{sec:exp4}. 
Our experiment is performed on an Ubuntu server with 4 NVIDIA GeForce 3090 GPUs. We used 0.6/0.2/0.2 as the train/validation/test split ratio, where the training and validation set is used to compute empirical quantile and conformal correction for constructing prediction intervals of SLCP in Eq. \eqref{eqn:localized_band}. 

\subsection{Methods} \label{sec:exp1}
We focus on regression problems where adaptive predictive interval (PI) is a requisite. Three regression models are used to evaluate each conformal prediction method: ridge regression \citep{zaikarina2016lasso}, random forests \citep{meinshausen2006quantile}, and neural networks \citep{taylor2000quantile}. Each model is adapted to both its mean and quantile regression variants. Parameters of ridge regression and random forest are determined using cross-validation. The expected coverage rate of all experiments is set to be 90\%.

\begin{table*}[t]
\centering
\renewcommand
\arraystretch{1}
\scalebox{1}{
\begin{tabular}{c|c|c|c|c|c|c|c|c}
\hlinewd{1.5pt}
\multicolumn{2}{c}{Model / Method} & SLCP & CQR & LCR & MAD-Split & Dist-Split & CD-Split & LVD \\ \hline
\multirow{3}{*}{ Ridge } & Length & \textbf{2.12} & 2.25 & 2.38 & 2.94 & 2.21 & \underline{2.16} & - \\
& Coverage & 90.82 & 89.8 & 89.59 & 90.13 & 90.02 & 89.85 & - \\
& Correlation & \textbf{0.45} & 0.39 & 0.41 & 0.35 & \underline{0.43} & 0.42 & - \\ \hline
\multirow{3}{*}{ Random Forest } & Length & \underline{1.38} & 1.41 & 1.63 & 1.82 & 1.43 & \textbf{1.36} & - \\
& Coverage & 90.32 & 90.33 & 90.12 & 89.95 & 90.36 & 90.27 & - \\
& Correlation & \underline{0.58} & 0.57 & 0.47 & 0.44 & 0.56 & \textbf{0.61} & - \\ \hline
\multirow{3}{*}{ Neural Nets } & Length & \textbf{1.38} & \underline{1.40} & 1.54 & 1.81 & 1.46 & 1.43 & 2.48 \\
& Coverage & 90.08 & 90.06 & 89.47 & 89.95 & 90.12 & 90.09 & 93.42 \\
& Correlation & \textbf{0.57} & \textbf{0.57} & 0.51 & 0.49 & 0.54 & \underline{0.56} & 0.38 \\
\hlinewd{1.5pt}
\end{tabular}
}
\caption{Average length (normalized), coverage rate (measuring \textit{average coverage}), and Pearson's correlation (measuring \textit{conditional coverage}) of prediction intervals constructed by different methods. For average length and Pearson's correlation, the best results are highlighted in bold font and the second-best results are highlighted in underline. All results are averaged across 13 real datasets and 20 random experiments.} 
\label{tab:real_results}
\end{table*}

\begin{table*}[t]
\centering
\renewcommand
\arraystretch{1.2}
\scalebox{1}{
\begin{tabular}{c|c|c|c|c|c}
\hlinewd{1.5pt}
Length / Cov. Rate & SLCP & CD-Split & CQR & MAD-Split & LVD \\ \hline
Overall (93627) & \textbf{4.65} / 90.03 & 12.79 / 90.15 & 23.14 / 91.24 & 16.26 / 91.12 & 32.47 / 94.45 \\ \hline
Age 0 - 20 (17423) & \textbf{3.75} / 90.04 & 8.63 / 90.23 & 23.95 / 91.37 & 20.39 / 92.24 & 38.64 / 95.49 \\ \hline
Age 21 - 60 (48815) & \textbf{4.63} / 90.19 & 14.82 / 90.56 & 22.98 / 91.19 & 15.43 / 90.89 & 31.23 / 94.22 \\ \hline
Age > 60 (27389) & \textbf{5.54} / 89.82 & 10.97 / 89.14 & 23.47 / 91.25 & 17.69 / 90.46 & 29.07 / 93.34 \\
\hlinewd{1.5pt}
\end{tabular}
}
\caption{Conditional coverage results of age regression task. The dataset can be divided into subgroups with different ages. Numbers within the bracket are the size of the subgroup. SLCP provides the most precise PIs with specified coverage rates.}
\label{tab:age_regression}
\end{table*}

\begin{table}
\centering
\renewcommand
\arraystretch{1}
\vspace{-1.1em}
\scalebox{0.95}{
\begin{tabular}{c|c|c|c|c} \hline
\multirow{2}{*}{ Data Types } & \multicolumn{2}{c}{Heteroscedastic} & \multicolumn{2}{c}{BiModal} \\ \hhline{~|----}
& Length & Coverage & Length & Coverage \\ \hline
LCP & 1.98 & 90.22 & 5.46 & 90.11 \\
MADSplit & 1.93 & 90.17 & 5.74 & 90.16 \\
CQR & 1.89 & 90.12 & 5.58 & 90.32 \\
Dist-Split & 2.09 & 90.56 & 5.63 & 90.38 \\
CD-Split & 1.91 & 90.13 & \textbf{5.45} & 90.15 \\
SLCP (Ours) & \textbf{1.82} & 90.08 & 5.54 & 90.36 \\\hline
\end{tabular}}
\caption{Comparison of PIs over two synthetic datasets.}
\label{tab:dist-split-results}
\end{table}

\paragraph{Baselines}
We evaluate SLCP along with multiple prominent methods that are most related to our work, including:
\begin{enumerate}
\item conformalized quantile regression (CQR) \citep{romano2019conformalized}, which constructs the nonconformity score using two quantile estimators, resulting in adaptive conformal intervals.
\item localized conformal prediction (LCP) \citep{guan2021localized}, which leverages local information by constructing an approximation of the empirical distribution.
\item locally adaptive conformal prediction (MADSplit) \citep{lei2018distribution}, which constructs adaptive conformal interval with an additional estimator (mean absolute deviation) to normalize the residuals.
\item locally valid and discriminative prediction intervals (LVD) \citep{lin2021locally}, which constructs discriminative prediction intervals using kernel regression.
\item Dist/CD-Split \citep{izbicki2019flexible}, which approximate conditional coverage using conditional density estimators \citet{izbicki2017converting}; CD-Split generates predictions based on local partition of feature space. 
\end{enumerate}

\paragraph{Evaluation Metrics}
We evaluate results based on metrics of both average coverage and conditional coverage. For average coverage, we employ the averaged length of prediction interval and coverage rate used by many recent methods \citep[e.g.][]{romano2019conformalized}. For conditional coverage, we use Pearson's correlation between the interval size and the indicator of coverage proposed by \citet{feldman2021improving}. In addition, for datasets with categorical variables, we partition the input space based on these variables and measure the coverage results for each subset.
Since SLCP can be built on both conditional mean and quantile estimators, we therefore assign SLCP to corresponding estimators when comparing with baselines with different regressors. 


\subsection{Synthetic Datasets} \label{sec:exp2}
To demonstrate the advantage of SLCP, we perform evaluations on multiple simulated datasets that contain heterogeneous noise and outliers. Figure \ref{fig:all_sim} shows some examples of how SLCP performs on these datasets. Figure (a) shows SLCP applied on the random forest mean estimator, which provides adaptive intervals that result in a shorter average interval length compared with vanilla split conformal prediction. In figure (b), a quantile random forest model is trained with SLCP, we observe that the generated PI can precisely adapt to the variation of data. We further study the behavior of SLCP when the regression model is less well specified to the underlying data distribution. As shown in figure (c), we fit a quantile linear regression model on a sinusoidal function with noise. Results show that SLCP can still produce adaptive intervals by incorporating localized information into the nonconformity score. This is a much-desired property given that model misspecification commonly happens in real-world complex datasets, which further supports the advantage of SLCP over other baselines.

We then quantitatively compare SLCP with baseline methods, using synthetic data from \citet{izbicki2019flexible}. The PIs are constructed upon quantile random forest estimators. As Table \ref{tab:dist-split-results} shows, SLCP can achieve better results on heteroscedastic data while closely approaching the results of CD-Split on bimodal data. We also observe that the LCP and CD-Split methods require more extensive computation than others, given that LCP is a full-conformal-based method and CD-Split needs to compute the profile distance.
\subsection{Real Datasets} \label{sec:exp3}
\paragraph{Tabular Data}
We now summarize the results on 13 real datasets, which are among the popular choices for regression as also used by \citet{romano2019conformalized} and \citet{lin2021locally}. We first normalize all features to be zero mean and unit variance and re-scale the response by its mean absolute values. The normalized data is then used to fit regression models and construct PIs by different conformal prediction methods. Table \ref{tab:real_results} shows the average performance across all datasets and random seeds after normalization. From the results, we observe that SLCP (quantile estimator) generally obtains the shortest PI given that the coverage rate has reached the desired level, resulting in better uncertainty characterization without being over-conservative. Specifically, SLCP provides better performance than other split-conformal-based methods: Dist-Split, CQR, and MADSplit given its ability to construct more robust PIs in different situations. CD-Split also achieves competitive performance, particularly on random forest models. For methods building upon neural network estimators, LVD provides much longer PIs given its sensitivity to outliers, which is normally unnecessary in many applications. Meanwhile, CQR combined with neural networks also demonstrates superior performance compared with other baselines.

\paragraph{Age Regression Data}
We also evaluate SLCP and other baselines under high-dimensional setting. To the best of our knowledge, experiments from all prior works focused on tabular data with very few variants. To address this limitation, we investigated an age regression problem featured in multiple datasets \citep{Rothe-ICCVW-2015, Rothe-IJCV-2018, zhifei2017cvpr}. We implement an age regression model based on Resnet-34 \citep{he2016deep} trained using MSE loss; the embedding obtained from the pre-trained model is then used to construct conformal PIs. Table \ref{tab:age_regression} shows coverage results using SLCP (mean estimator) and other baselines, SLCP provides the most accurate PIs while satisfying the prespecified coverage rate. Note that in this experiment, CQR shows sub-optimal performance than other datasets mainly because of the employment of the quantile loss function in training deep neural networks. Given that quantile loss is an $L_1$ loss function, it is neither effective nor efficient in training large-scale models.

\subsection{Conditional Coverage Results} \label{sec:exp4}
Table 1 shows that SLCP can also achieve better conditional coverage from its higher Pearson's correlation. In addition, for datasets with categorical variables, we leverage this information to partition the input space into multiple subgroups. Table \ref{tab:age_regression} also shows conditional coverage results on the age regression task divided into multiple age groups. Additional results on other tabular data can be found in Appendix. 
Lastly, readers can find detailed ablation studies on multiple design choices (e.g., batch size and kernel density bandwidth) that are mentioned in section \ref{sec:slcp}. We also investigate how changes in regression model capacity and covariate shift on test data can affect SLCP and other baseline methods. Results show that SLCP is more robust over baselines under different situations.

\section{Concluding Remark}
In this paper, we propose a simple and effective approach SLCP that constructs the non-conformity score by leveraging localized information. Compared with other conformal prediction methods, SLCP can approximate conditional coverage guarantee while satisfying average coverage guarantee, and is robust under circumstances such as heterogeneous data and mis-specified models. Besides, we also proposed a unified conformal prediction framework that includes multiple state-of-the-art methods. Empirical results on image and tabular datasets show that SLCP can achieve better conditional coverage and more robust to covariate shift compared with previous methods. As far as our concern, this work does have positive social impacts as (1) the improvement of predictive intervals makes an important contribution on reliable decision making; (2) the approximated conditional coverage guarantee provides better coverage for each subgroup of data, which can be particularly helpful for fairness related problems. As for a future work, we plan to investigate application-specific conformal prediction methods incorporated with different localized approximation strategies under the proposed unified framework.

\bibliography{references}

\onecolumn
\appendix
\begin{center}
\Large
\textbf{Supplementary Material for \\ Split Localized Conformal Prediction}
\end{center}

\section{Asymptotic Conditional Coverage}
\begin{pro}
As $n\rightarrow \infty$, the PI in Eq. \eqref{eqn:SLCP_band} achieves asymptotic conditional coverage:
\begin{equation}
    \P(Y_{n+1}\in \hat{\C}^{\rm SLCP}(X_{n+1})|X_{n+1}) \geq \alpha.
\end{equation}
\end{pro}

\begin{proof}
Denote $\hat{p}_h(x)$ as the empirical distribution estimated by $K(\|X_i -  X\| / h)$, $p(x)$ as the original distribution, and $\E[\hat{p}_h(x)] = p_h(x)$. Assume $X$ has $d$-dimensional features and $h$ is the bandwidth of kernel density estimation. Let $L$ be positive numbers and $\beta$ be positive integers. Given a vector $s=(s_1, \dots, s_d)$, define $|s|=s_1+\dots+s_d$, $x^s=x_1^{s_1} \cdots x_d^{s_d}$ and $D^s = \frac{\partial^{s_1+\dots+s_d}}{\partial x_1^{s_1} \cdots \partial x_d^{s_d}}$. Define the H$\mathrm{\ddot{o}}$lder class as 
\begin{equation}
    \Sigma(\beta, L) = \left\{ g: |D^s g(x) - D^s g(y)| \leq L \|x-y\|, \mathrm{\forall s ~such~that~|s|=\beta-1, ~and~all~}x,y \right\}.
\end{equation}
We first bound the bias of $\hat{p}_h(x)$:
\begin{align}
    |p_h(x)-p(x)| &= \int \frac{1}{h^d} K(\|u-x\|/h) p(u)du - p(x) \notag \\
    &= \left| \int K(\|v\|) (p(x+hv) - p(x)) dv\right| \notag \\
    &\leq \left| \int K(\|v\|) (p(x+hv) - p_{x, \beta}(x+hv))dv\right|+\left| \int K(\|v\|)(p_{x, \beta}(x+hv) - p(x))dv\right|. \notag
\end{align}
The first term is bounded by $Lh^{\beta}\int K(s)|s|^{\beta}$ since $p\in \Sigma(\beta, L)$. The second term is $0$ from the properties on $K$ since $p_{x, \beta}(x+hv) - p(x)$ is a polynomial of degree $\beta$ (without constant term). Then for some $c$, the bias of $\hat{p}_h(x)$ satisfies
\begin{equation}
    \underset{p\in \Sigma(\beta, L)}{\sup}~ |p_h(x) - p(x)| \leq ch^{\beta}.
    \label{eq:bias}
\end{equation}
We write $\hat{p}(x)=n^{-1} \sum_{i=1}^n Z_i$, where $Z_i = \frac{1}{h^d}K\left(\frac{\|x-X_i\|}{h} \right).$ Then
\begin{align}
    \mathrm{Var}(Z_i) &\leq \E(Z_i^2) = \frac{1}{h^2d} \int K^2\left( \frac{\|x-u\|}{h}\right)p(u) du = \frac{h^d}{h^{2d}} \int K^2(\|v\|) p(x+hv) dv \notag \\
    &\leq \frac{\sup_x p(x)}{h^d} \int K^2(\|v\|) dv \leq \frac{c}{h^d}, \notag
\end{align}
for some $c > 0$, since the densities in $\Sigma(\beta, L)$ are uniformly bounded. Then the variance of $\hat{p}_h(x)$ satisfies
\begin{equation}
\underset{p\in \Sigma(\beta, L)}{\sup} \mathrm{Var}(\hat{p}_h(x)) \leq \frac{c}{nh^d}.
\label{eq:var}
\end{equation}
Combining Eq. \eqref{eq:bias} and \eqref{eq:var}, if $h\asymp n^{-1/(2\beta+d)}$ then we have
\begin{equation}
    \underset{p\in \Sigma(\beta, L)}{\sup} \E\int (\hat{p}_h(x) - p(x))^2 dx \preceq \left( \frac{1}{n}\right)^{\frac{2\beta}{2\beta+d}}.
\end{equation}
Plug into the approximation of conditional CDF function $\hat{F}_h$, we have
\begin{align}
    \E \int \left[\hat{F}_h(V|X=x)-F(V|X=x) \right]^2 dx &= \sum_{i \in \mathcal{I}_{\mathrm{train}}}\left[\omega_h(X_i|x) - \omega(X_i|x))\right]^2 \notag  \\
    &\leq  \underset{p\in \Sigma(\beta, L)}{\sup} \E\int (\hat{p}_h(x) - p(x))^2 dx \notag \\
    &\preceq \left( \frac{1}{n}\right)^{\frac{2\beta}{2\beta+d}},
\end{align}
which means if a sufficient amount of data is available, i.e. $n\to \infty$, our non-parametric estimation of conditional distribution $\hat{F}_h(V|X=x)$ will converge to the true conditional distribution $F(V|X=x)$. In this case, SLCP will provide a new non-conformity score $V^{\alpha,h} \approx V(x,y) - q_{v,\alpha}(x)$,
where $q_{v,\alpha}(x)$ denotes the true quantile function $q_{v,\alpha}(x) = \inf\{v: F(v|X=x)\geq \alpha\}$.
Therefore, $\P(V^{\alpha,h}\geq 0|X=x) \approx \alpha$ is uniform for all $x$. The marginal quantile correction of $V^{\alpha,h}$ will be close to $0$ and provide an asymptotic conditional coverage guarantee.
\end{proof}

\section{Discussion on the Generalized Framework}
\label{sec:appendix_generalize_framework}

\begin{thm}[Generalized SLCP]
    Suppose $m(v, F, \cdot)$ is a score function with its inverse function $p(v,F,\cdot)$. 
    Let $\theta_i = p(V_i, \hat{F}_i, 0)$, and suppose $\theta^* = \Q(\alpha, \{\theta_i\}\cup \{\infty\})$
    we have:
    \begin{equation}
        \Pr\{m(V_{n+1}, \hat{F}_{n+1}, \theta^*) \leq 0\} \geq \alpha.
    \end{equation}
    Moreover, let $\hat{\C}(x) = \{y: m(v, \hat{F}, \theta^*)\leq 0\}$, where $v$ is the non-conformity score for $V(x,y)$, and $\hat{F}$ is the approximate CDF conditioned on $x$, we have:
    \begin{equation}
        \Pr\{Y_{n+1} \in \hat{\C}(X_{n+1})\} \geq \alpha.
    \end{equation}
\end{thm}

\begin{proof}
    Suppose $E_v$ is the event when $\{V_1, V_2,\ldots, V_{n}\} = \{v_1, v_2,\ldots, v_{n}\}$, here we assume $v_i$ is unique almost surely.
    Given $\theta$, denote $M_i^{\theta} = m(V_i, \hat{F}_i, \theta)$ and $m_i^{\theta}$ as its realization when $V_i = v_i$.
    By Quantile Lemma, we have:
    \begin{align*}
        \Pr[M_{n+1}^{\theta}\leq \Q(\alpha, \{m_{1:n}^{\theta}\}\cup \{\infty\})|E_v, \theta] \geq \alpha.
    \end{align*}

    Let $g(\theta) = \Q(\alpha, \{m_{1:n}^{\theta}\}\cup \{\infty\})$.
    By definition of $\theta^*$, $\theta^*$ is less or equal to at least the $\alpha$-quantile of all $\theta_i$ (plus $\{\infty\}$), by the monotonicity, and the definition of $\theta_i$ that $m_i^{\theta_i} = 0$, we know at least $\alpha$-quantile of $m_i^{\theta^*}\leq 0$.
    Therefore we have $g(\theta^*) \leq 0$.
    
    Since $\theta^*$ depends on the value of $\theta_i$ so $\theta^*$ is fixed given event $E_v$, thus
    $$
    \Pr\{M_{n+1}^{\theta^*} \leq g(\theta^*)\leq 0|E_v\} \geq \alpha.
    $$
    Marginalized for every possible event $E_v$, we have
    $$
    \Pr\{m(V_{n+1}, \hat{F}_{n+1}, \theta^*) \leq 0\} \geq \alpha.
    $$
\end{proof}

\paragraph{Other Options of Score Function $m$}
We list a number of potential monotonic score functions and their inverse functions.
Readers are encouraged to build their own score functions for different applications.
\begin{enumerate}
    \item Additive with Quantile (our method, SLCP): $m(v,F,\theta) = \Q(\alpha,F) + \theta - v$. Inversion: $p(v,F,0) = v - \Q(\alpha,F)$.
    \item Multiplicative with Quantile: $m(v,F,\theta) = \theta(\varepsilon+|\Q(\alpha,F)|) - v , \theta > 0$. Inversion: $p(v,F,0) = \frac{v}{(\varepsilon+|\Q(\alpha,F)|)}$.
    \item Additive with Expectation: $m(v,F,\theta) = \E_F[V] + \theta - v$. Inversion: $p(v,F,0) = v - \E_F[V]$.
    \item $\theta$-Quantile~(Guan, 2019):
    $m(v,F,\theta) = Q(\theta, F) - v$. Inversion: $p(v,F,0) = \Pr_F[V<v]$.
\end{enumerate}

\section{Additional Analysis on Different Scenarios}

\paragraph{When $V$ is Independent to $X$}
If $V$ is (approximately) independent of $X$, 
then split conformal prediction can also provide the conditional guarantee as we discussed in Section 3.1.
However, the use of SLCP does not hurt the independence. Since each $V(x,y)$ is not dependent on $X=x$,
the weighted average of $V$ is also independent to $X$. Therefore, the quantile value of $\hat{F}_h$ is also independent to $X$.
To summarize, the new score $V^{\alpha,h}$ is independent to $X$ no matter what kernel or bandwidth we are using; similar to split conformal prediction, the final prediction intervals obtained also satisfies conditional guarantee.

\paragraph{When $V$ is Residual from True Quantiles.}
Suppose quantile regression is used to estimate the conditional quantile model $q_{Y,\alpha}(X)$
and a precise model is obtained to construct $V = Y - q_{Y,\alpha}(X)$, we will have the following results.
\begin{thm}\label{thm:doubly_robust}
Denote $F_h(v|X=x)$ as the population version of $\hat{F}_h$,
$$
 F_h(v|X=x) = \E[w_h(X|x)\mathds{1}_{V\leq v}],
$$
where $w_h(X|x) = \frac{K(\|X-x\|/h)}{\E_X K(\|X-x\|/h)}$.
If $V = Y - q_{Y,\alpha}(X)$
we will have
$$
    \Q(\alpha, F_h(V|X=x)) = 0
$$
\end{thm}
\begin{proof}
    If $V = Y - q_{Y,\alpha}(X)$, then 
    \begin{align*}
        \Pr[Y-q_{Y,\alpha}(x)\leq 0|X=x] = \alpha,
    \end{align*}
    which implies:
    \begin{align*}
        \int_{\X} p_X(z) w_h(z|x)\Pr[Y-q_{Y, \alpha}(x)\leq 0|X=x] d z = \alpha.
    \end{align*}
    By the definition of quantile we can see that the minimum value of $v$ that satisfies $\Pr[V = Y-q_{Y,\alpha}(X)\leq v]$
    under the distribution $F_h$ is $0$, therefore
    $$
    \Q(\alpha, F_h(V|X=x)) = 0.
    $$
\end{proof}

This is because the probability of $V$ is greater than $0$ is exactly $\alpha$ for any $V$.
Therefore, $0$ is the exact $\alpha$-quantile for any $V|X=x$, 
and also the $\alpha$-quantile for $F_h$. This indicates the empirical quantile is close to $0$,
and therefore $V^{\alpha,h} \approx V$, which satisfies the conditional coverage guarantee.

\section{Ablation Studies}
\label{sec:ablation}
We conduct additional ablation studies on SLCP baselines by varying the representation power of neural network models as well as the degree of covariate shift in the test set.

\paragraph{Model Capacity} We change the hidden size of neural network models and measure the average length of predictive intervals across all datasets. Figure \ref{fig:ablation} (a) shows that when using neural network with small hidden size, SLCP (quantile estimator) has a clear advantage over split conformal and CQR. After the hidden size is above 32, the performance of CQR has almost no difference with SLCP. The change of hidden size affects the results of vanilla split conformal prediction similarly, with a larger average length in general. 

\paragraph{Covariate Shift} The change of distribution in certain features is likely to hamper conformal predictions. In the ideal case when the conditional coverage is perfectly achieved, covariate shift does not affect the performance. But in real situation, the regression models used to approximate conditional coverage cannot fully capture the data distribution. We conduct a simulation where the training and testing data are drawn from $\mathrm{Gaussian}$ and $\mathrm{Beta}$ distributions respectively. We change the parameters of $\mathrm{Beta}$ distribution to control the level of covariate shift. Figure \ref{fig:ablation} (b) shows that SLCP is more robust to covariate shift, by adding localized information on top of regression models.

\paragraph{Mini-batch Size}
We conducted additional experiments on the effect of batch size on the performance of all SLCP methods. Results are shown in Table \ref{tab:batch}. The first three rows (SLCP-KNN, SLCP-RBF, and SLCP-mean) are SLCP performed on quantile estimators using boxcar kernel, SLCP performed on quantile estimators using Gaussian-RBF kernel, and SLCP performed on mean estimators with the Gaussian-RBF kernel. All results are evaluated on 16 tabular datasets averaged over 20 random experiments as in table 1. We also include results running on the IMDB-WIKI dataset from the age regression task (averaged across 5 random runs). It is true that a very small batch size (<10\% of training set size) will lead to an obvious performance drop given that a specific mini-batch is more likely to omit the neighboring examples. From the ablation study, we empirically find that using 30\% of the training set as mini-batch size can significantly reduce computation time without hampering too much on prediction intervals of SLCP.

\paragraph{Kernel Density Bandwidth}
The other ablation study is on the bandwidth of the Gaussian RBF kernel at different quantiles of pairwise distance. As mentioned in section 3.4, we choose the bandwidth as the median value of the pairwise distance, which is the 0.5 quantile in the table. As shown in Table \ref{tab:bandwidth}, this leads to better results compared with all other quantiles in both tabular and image datasets.

\begin{figure*}[t]
    \begin{minipage}{0.85\textwidth}
    \centering
    \begin{tabular}{@{\hspace{7ex}} c @{\hspace{-4ex}} c}
        \begin{tabular}{c}
        \includegraphics[width=.5\textwidth]{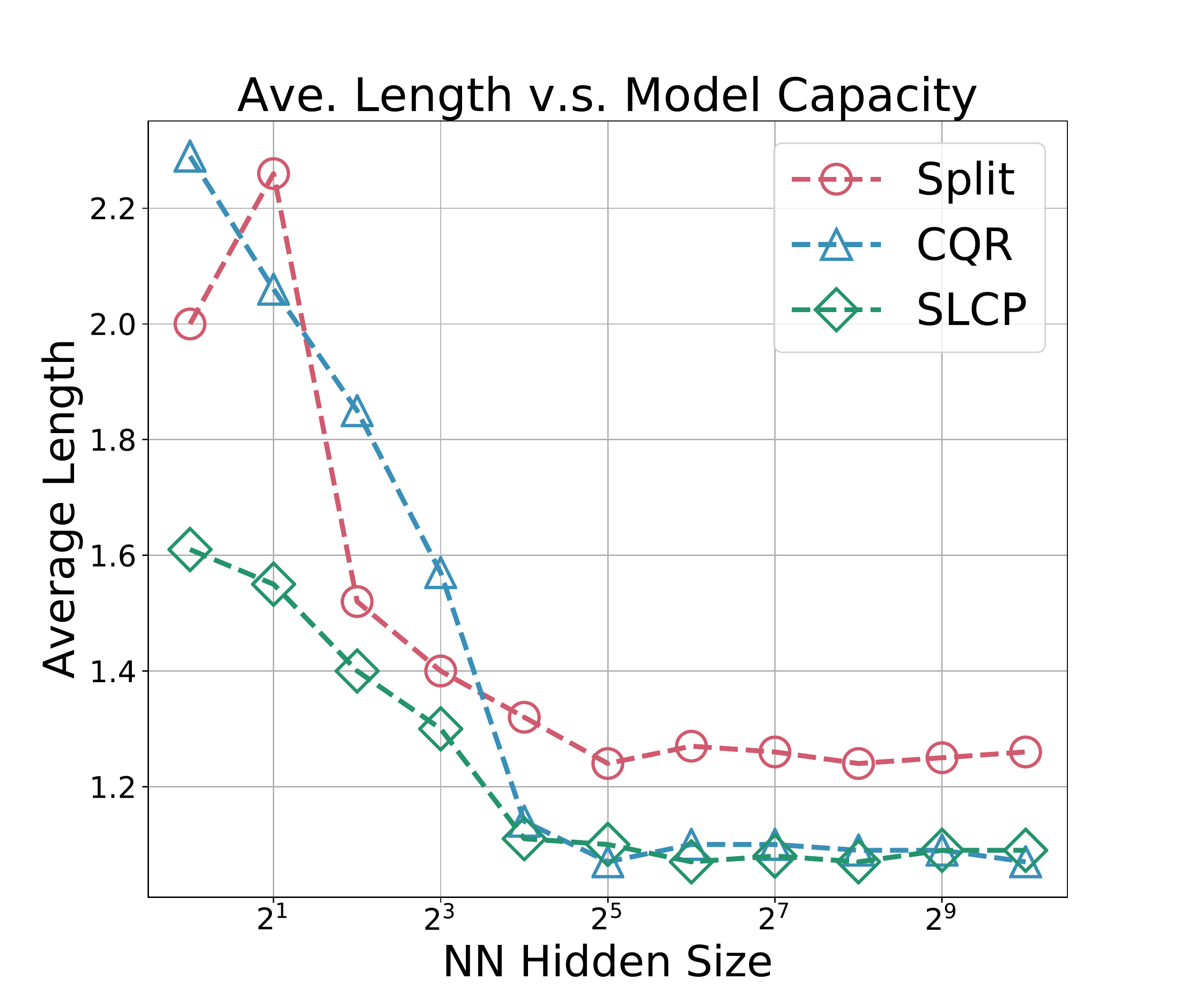}
        \\
        {\small{(a)}}
        \end{tabular} &
        \begin{tabular}{c}
        \includegraphics[width=.5\textwidth]{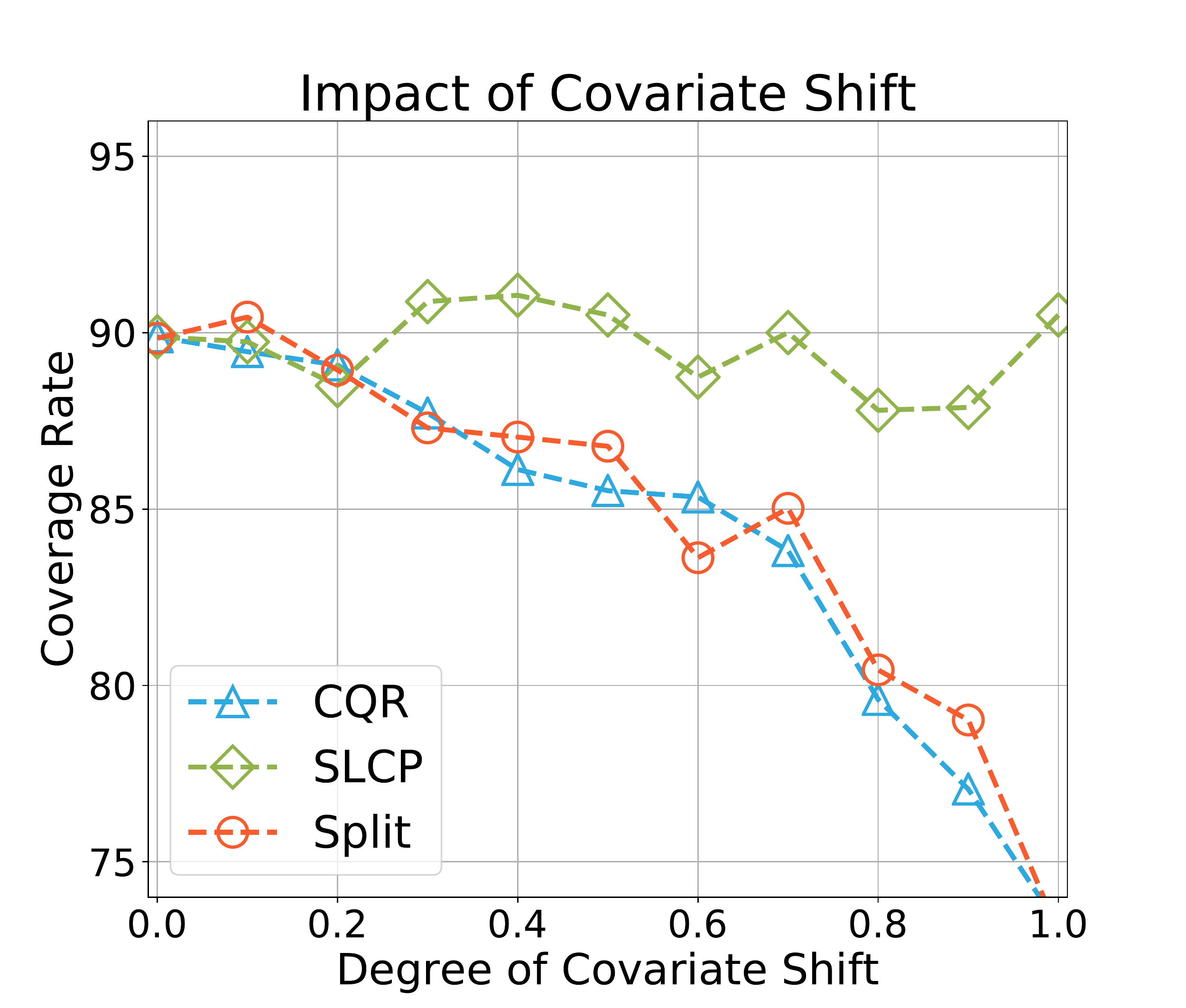}
        \\
        {\small{(b)}}
        \end{tabular} \\
        \end{tabular}
    \end{minipage}
    \vspace{-2ex}
    \caption{
    (a) Relationship between model capacity and average length of PIs. SLCP is more robust to model mis-specification. (b) Impact of testing data covariate shift on coverage rate of each method.}
    \label{fig:ablation}
\end{figure*}

\begin{table*}[t]
\centering
\renewcommand
\arraystretch{1.2}
\scalebox{0.95}{
\begin{tabular}{c|c|c|c|c|c|c|c}
\hlinewd{1.5pt}
& & \multicolumn{6}{c}{Batch Size (\% of training set)} \\ \hhline{~|~|-|-|-|-|-|-}
& & 10 & 30 & 50 & 70 & 90 & 100 \\ \hline
\multirow{2}{*}{ SLCP-KNN } & Ave. Length & 1.86 & 1.86 & 1.84 & 1.86 & 1.84 & 1.85 \\
& Ave. Coverage & 89.48 & 89.26 & 89.22 & 89.48 & 89.26 & 89.36 \\\hline
\multirow{2}{*}{ SLCP-RBF } & Ave. Length & 2.19 & 1.64 & 1.56 & 1.54 & 1.59 & 1.58 \\
& Ave. Coverage & 93.44 & 90.44 & 89.86 & 89.56 & 90.46 & 90.14 \\\hline
\multirow{2}{*}{ SLCP-mean } & Ave. Length & 2.24 & 2.17 & 2.16 & 2.18 & 2.18 & 2.17 \\
& Ave. Coverage & 91.16 & 90.74 & 90.5 & 91.06 & 91.02 & 90.76 \\\hline
\multirow{2}{*}{ Age Regression } & Ave. Length & 8.47 & 4.65 & 4.78 & 4.62 & 4.39 & 4.42 \\
& Ave. Coverage & 91.26 & 90.23 & 90.45 & 90.21 & 90.15 & 90.86 \\\hline
\multicolumn{2}{c}{Average Relative Run Time} & 1 & 3.46 & 5.76 & 8.29 & 11.09 & 12.42 \\
\hlinewd{1.5pt}
\end{tabular}}
\caption{Effect of mini-batch size on the performance of all SLCP methods.}
\label{tab:batch}
\end{table*}

\begin{table*}[t]
\centering
\renewcommand
\arraystretch{1.2}
\scalebox{0.98}{
\begin{tabular}{c|c|c|c|c|c|c|c}
\hlinewd{1.5pt}
& & \multicolumn{6}{c}{Bandwidth of RBF Kernel (quantile of pairwise distance)} \\ \hhline{~|~|-|-|-|-|-|-}
& & 0.1 & 0.3 & 0.5 & 0.7 & 0.9 & 1 \\ \hline
\multirow{2}{*}{ Tabular Data } & Ave. Length & 1.87 & 1.64 & \textbf{1.38} & 1.69 & 1.69 & 1.65 \\
& Ave. Coverage & 92.4 & 89.36 & 90.08 & 90.12 & 90.12 & 89.61 \\ \hline
\multirow{2}{*}{ Age Regressiom } & Ave. Length & 6.22 & 4.93 & \textbf{4.65} & 4.69 & 4.74 & 4.62 \\
& Ave. Coverage & 91.65 & 90.32 & 90.30 & 90.26 & 90.28 & 88.34 \\
\hlinewd{1.5pt}
\end{tabular}}
\caption{Ablation study on the bandwidth of the Gaussian RBF kernel at different quantiles of pairwise distance.}
\label{tab:bandwidth}
\end{table*}

\section{Dataset Information}
Below is the information for all tabular dataset (sample size, ambient dimension):
\begin{itemize}
    \item \texttt{Simulations} (7000, 1): simulation function 1: $y = (\sin{x})^2 + 0.6 \cdot \sin{2x} + \epsilon$; simulation function 2: $y = 2\cdot(\sin{x})^2 + 0.15 \cdot x \cdot \epsilon$; simulation function 3: $y = \mathrm{Pois}[(\sin{x})^2 + 0.1] + 0.08 \cdot x \cdot \epsilon + 25 \cdot \mathbbm{1}\{\varepsilon < 0.01\} + \epsilon$, where $\epsilon$ is a random variable that generates either $0$ or $1$, and $\varepsilon \sim \mathrm{Unif}[0, 1].$
    \item \href{https://archive.ics.uci.edu/ml/datasets/Bike+Sharing+Dataset}{\texttt{bike}} (10886, 12): UCI bike sharing dataset.
    \item \href{https://archive.ics.uci.edu/ml/datasets/Physicochemical+Properties+of+Protein+Tertiary+Structure}{\texttt{bio}} (45730, 10): UCI physicochemical properties of protein tertiary structure dataset.
    \item \href{https://www.rdocumentation.org/packages/AER/versions/1.2-6/topics/STAR}{\texttt{star}} (11598, 48): Tennessee’s Student Teacher Achievement Ratio (STAR) project.
    \item \href{http://archive.ics.uci.edu/ml/datasets/concrete+compressive+strength}{\texttt{concrete}} (1030, 9): concrete compressive strength dataset.
    \item \href{http://archive.ics.uci.edu/ml/datasets/communities+and+crime}{\texttt{community}} (1993, 128): UCI communities and crime data set.
    \item \href{https://meps.ahrq.gov/mepsweb/data_stats/download_data_files_detail.jsp?cboPufNumber=HC-181}{\texttt{meps}} (15785, 141): medical expenditure panel survey, including panel 19, 20 and 21. 
    \item \href{https://archive.ics.uci.edu/ml/datasets/Facebook+Comment+Volume+Dataset}{\texttt{facebook}} (40948, 54): UCI facebook comment volume data set.
    \item \href{http://archive.ics.uci.edu/ml/datasets/yacht+hydrodynamics}{\texttt{yacht}} (308, 7): UCI yacht hydrodynamics dataset.
    \item \href{https://archive.ics.uci.edu/ml/datasets/energy+efficiency}{\texttt{energy}} (768, 8): UCI energy efficiency dataset.
    \item \href{https://www.openml.org/search?type=data&sort=runs&id=189&status=active}{\texttt{kin8nm}} (8192, 9): Forward kinematics of an 8 link robot arm.
\end{itemize}

Age regression data:
\begin{itemize}
    \item \href{https://data.vision.ee.ethz.ch/cvl/rrothe/imdb-wiki/}{IMDB-WIKI}: this dataset contains 500k+ face images with age and gender labels. We use the face-only cropped image (8GB) for the experiment.
    \item \href{https://susanqq.github.io/UTKFace/}{UTK Face}: this dataset contains 20k+ face images with annotations of age (range from 0 to 116), gender, and ethnicity. We use the aligned \& cropped faces (107MB) for experiment.
\end{itemize}
All datasets are publicly available on the web.


\section{Additional Results} \label{sec:additional_results}
\begin{figure*}[t!]
    \begin{minipage}{\textwidth}
    \centering
    \begin{tabular}{@{\hspace{-2.6ex}}c @{\hspace{-2.6ex}} c @{\hspace{-2.6ex}} c}
        \begin{tabular}{c}
        \includegraphics[width=.32\textwidth]{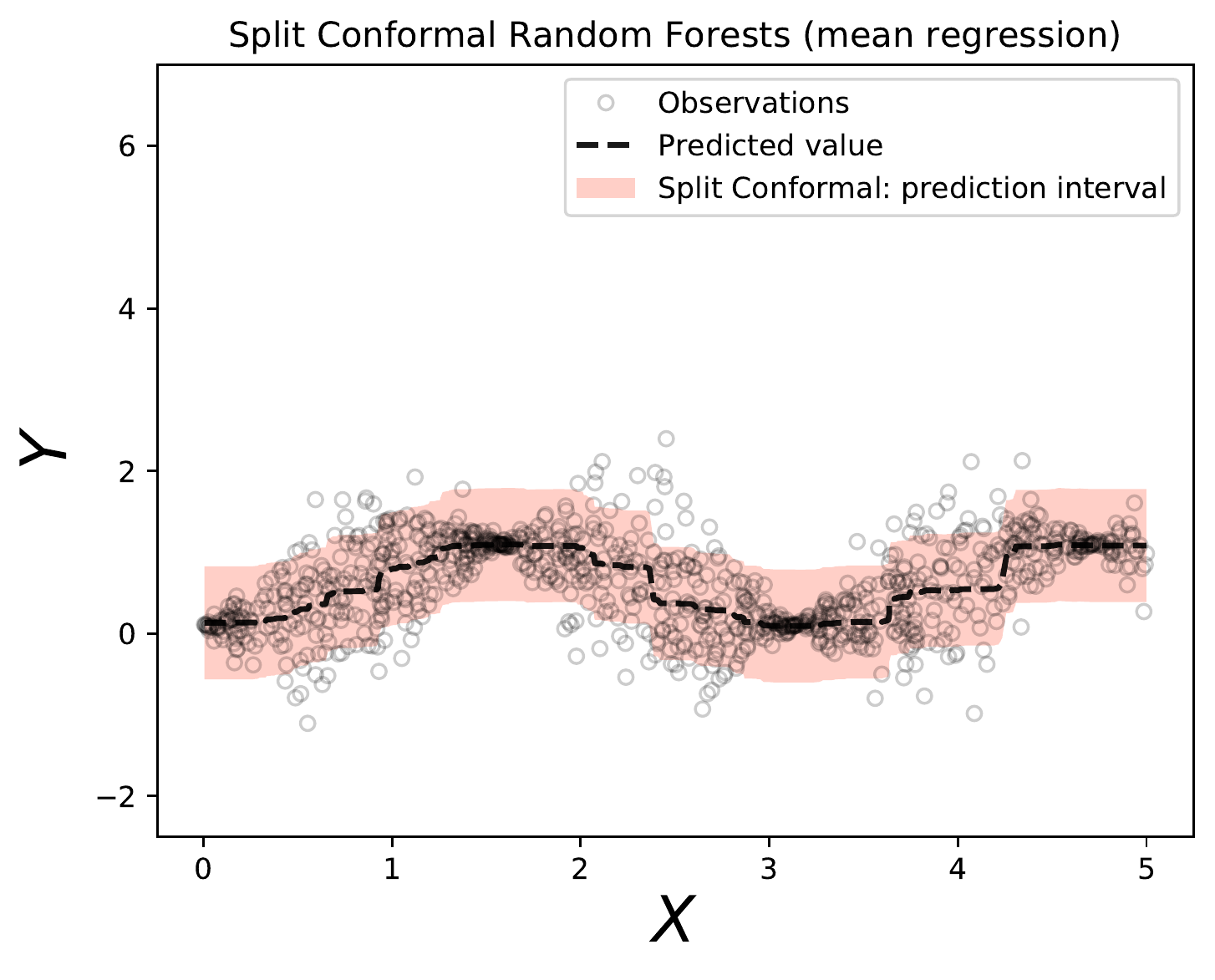}
        \\
        {\small{(a) Avg. cov. 90.03\%; Avg. len. 1.39}}
        \end{tabular} &
        \begin{tabular}{c}
        \includegraphics[width=.32\textwidth]{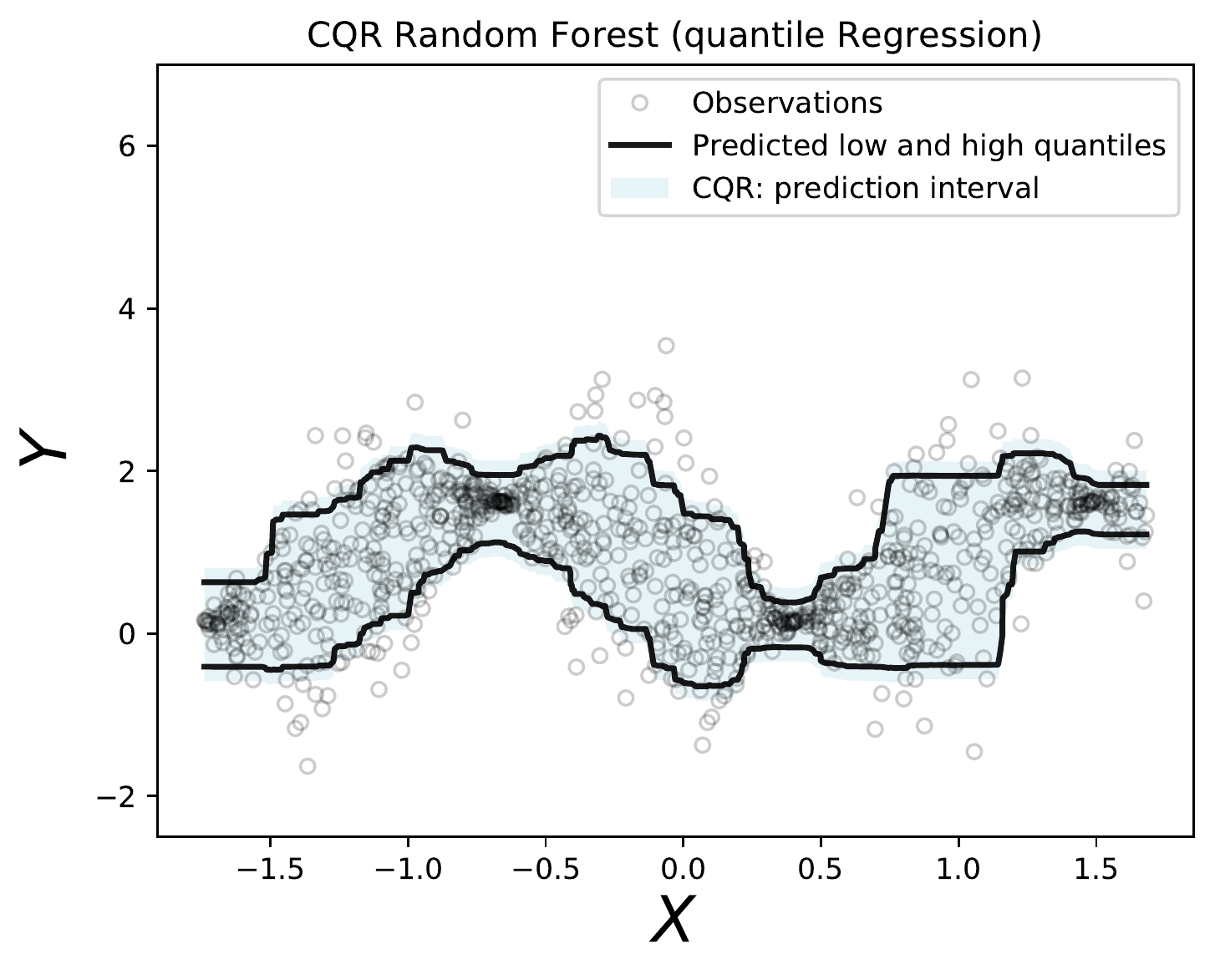}
        \\
        {\small{(c) Avg. cov. 89.84\%; Avg. len. 1.86}}
        \end{tabular} &
        \begin{tabular}{c}
        \includegraphics[width=.32\textwidth]{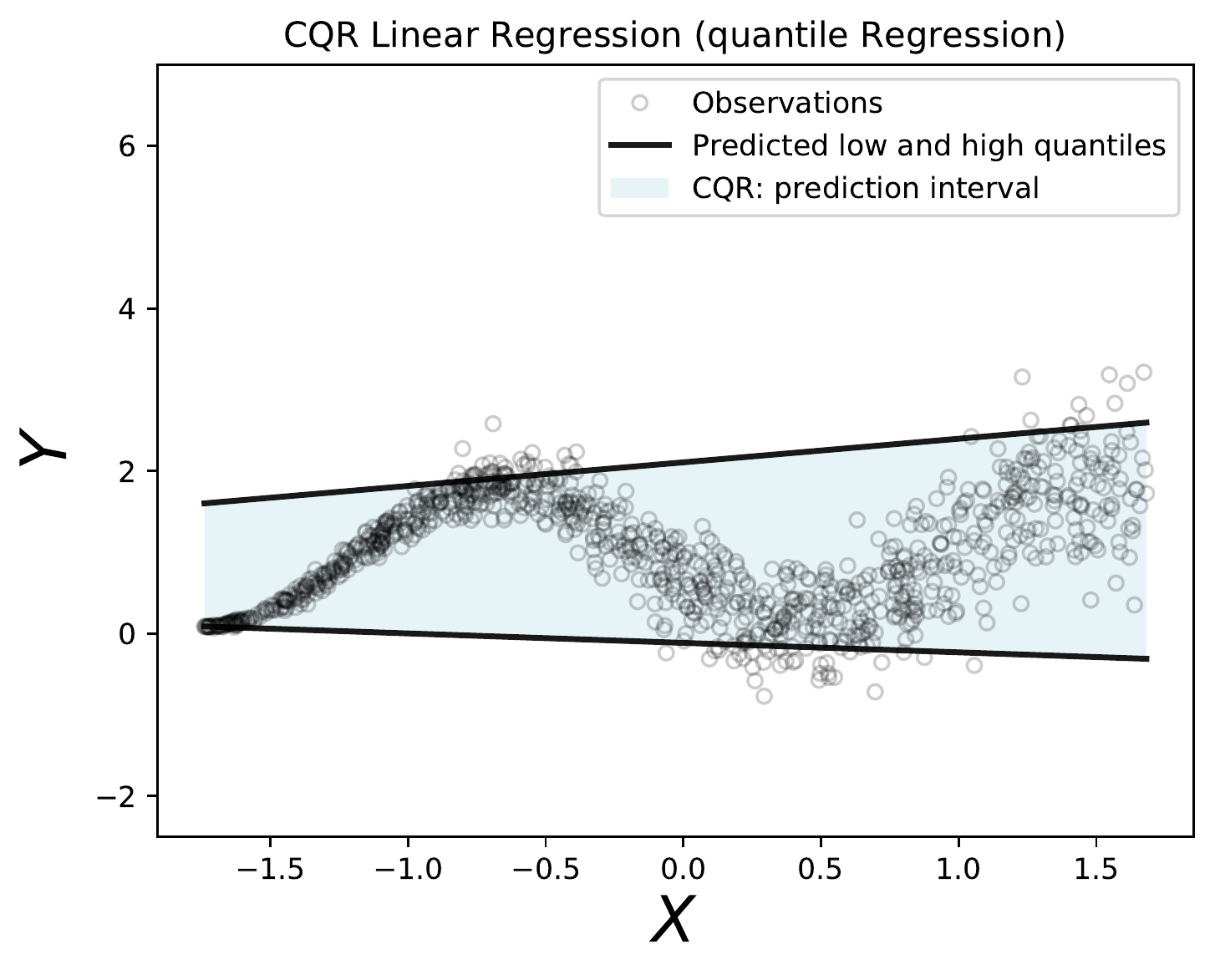}
        \\
        {\small{(e) Avg. cov. 89.48\%; Avg. len. 2.18}}
        \end{tabular} \\
        \begin{tabular}{c}
        \includegraphics[width=.32\textwidth]{fig/SLCP_mean.pdf} 
        \\
        {\small{(b) Avg. cov. 90.36\%; Avg. len. 1.27}}
        \end{tabular} &
        \begin{tabular}{c}
        \includegraphics[width=.32\textwidth]{fig/SLCP_RF.pdf}
        \\
        {\small{(d) Avg. cov. 90.06\%; Avg. len. 1.86}}
        \end{tabular} &
        \begin{tabular}{c}
        \includegraphics[width=.32\textwidth]{fig/SLCP_linear.pdf} 
        \\
        {\small{(f) Avg. cov. 90.52\%; Avg. len. 1.52}}
        \end{tabular} \\
        \end{tabular}
    \end{minipage}
    \vspace{-1.5ex}
    \caption{Comparison of conformal prediction methods on simulated heteroscedastic data with outliers. Figures (a) and (b) compare split conformal prediction and SLCP using a random forest mean estimator. Figures (c) and (d) show the coverage of CQR and SLCP using a random forest quantile estimator. Figures (e) and (f) show the difference between SLCP and CQR when the model is misspecified. The PIs constructed by SLCP attains better coverage and robust in different situations.}
    \label{fig:total_sim}
\end{figure*}

Here, we further demonstrate the advantage of SLCP using synthetic data. We compare SLCP with vanilla split conformal prediction, where a random forest mean estimator is trained; results are shown in figures \ref{fig:total_sim} (a) and (b). Although both methods satisfy the coverage guarantee, SLCP provides more adaptive intervals that result in a shorter average interval length. On the contrary, split conformal prediction generates a fixed length interval, which is too conservative at multiple places. Figures \ref{fig:total_sim} (c) and (d) compare SLCP with CQR: we train a quantile random forest model for each method using the same set of hyper-parameters. SLCP demonstrates almost the same results as CQR, given the quantile regression model has precisely captured the upper and lower levels of the region. However, the superiority of SLCP is more prominent if the quantile regression model is less well-specified. Figures (e) and (f) show results of the same method when replacing quantile random forest with quantile linear regression, i.e., the underlying model assumption is incorrect. PIs generated by CQR rely heavily on quantile regression models, while SLCP can still produce adaptive intervals by incorporating localized information into the nonconformity score.

Figure \ref{fig:slcp-knn_neural_net_sim} to figure \ref{fig:cqr_neural_net_sim} show simulated results of SLCP using different kernels: boxing kernel (KNN), RBF, and Epanechnikov (EPA) kernels. We finally adapt the RBF kernel in our experiments. 
In addition, table \ref{tab:star} to table \ref{tab:meps} show conditional coverage results of tabular datasets on each subgroup partitioned by categorical variables in the dataset. All methods are trained on the same neural network models as in the main paper. Results show that SLCP demonstrates better performance.

\begin{figure*}[t!]
    \begin{minipage}{\textwidth}
    \centering
    \begin{tabular}{c @{\hspace{-2.6ex}} c @{\hspace{-2.6ex}} c}
        \begin{tabular}{c}
        \includegraphics[width=.32\textwidth]{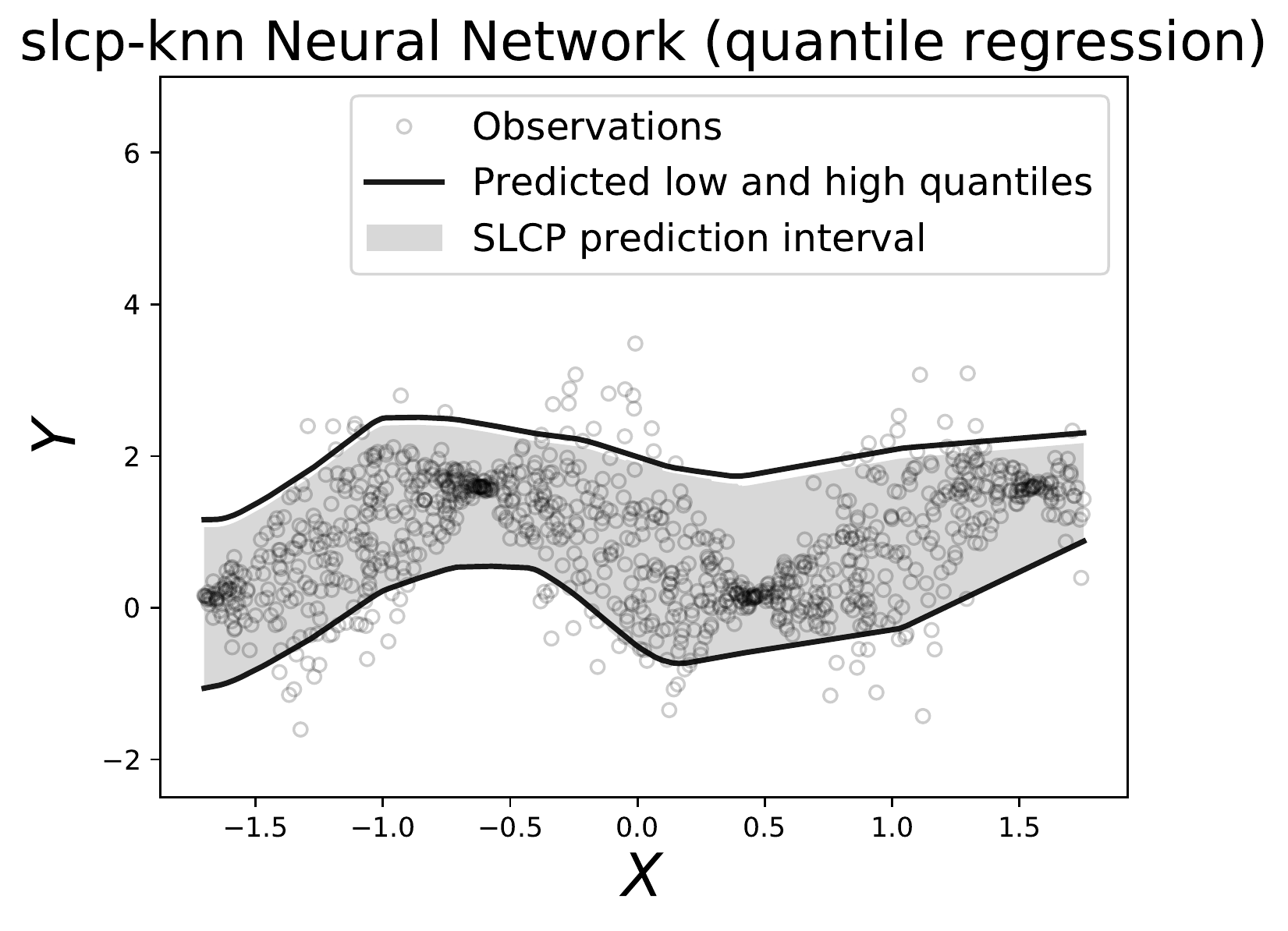}
        \\
        {\small{(a) Avg. cov. 89.46\%; Avg. len. 2.04}}
        \end{tabular} &
        \begin{tabular}{c}
        \includegraphics[width=.32\textwidth]{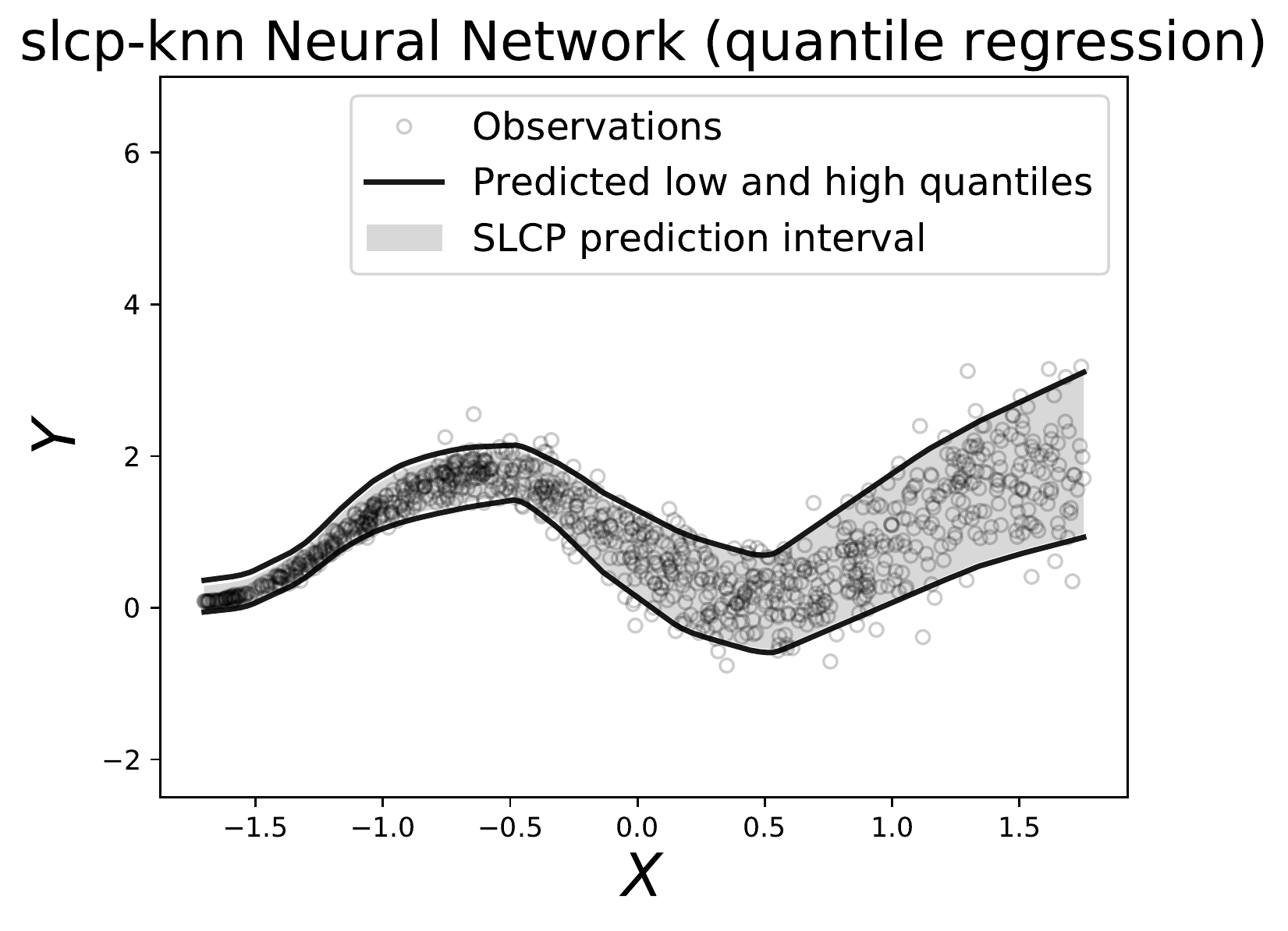}
        \\
        {\small{(b) Avg. cov. 89.32\%; Avg. len. 1.04}}
        \end{tabular} &
        \begin{tabular}{c}
        \includegraphics[width=.32\textwidth]{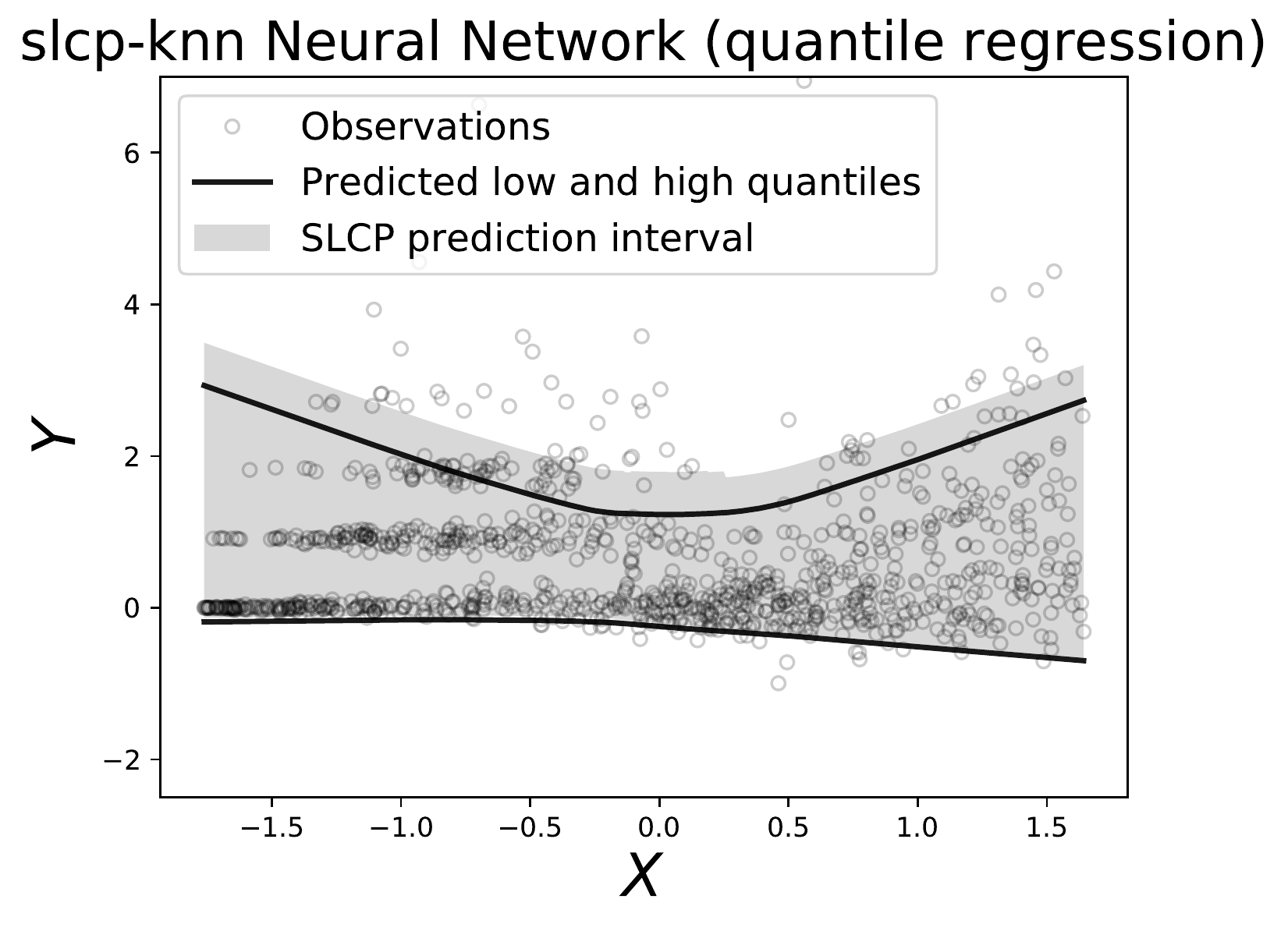}
        \\
        {\small{(c) Avg. cov. 90.70\%; Avg. len. 2.66}}
        \end{tabular} \\
        \end{tabular}
    \end{minipage}
    \vspace{-2ex}
    \caption{SLCP-knn on quantile neural network estimator.}
    \label{fig:slcp-knn_neural_net_sim}
\end{figure*}

\begin{figure*}[t!]
    \begin{minipage}{\textwidth}
    \centering
    \begin{tabular}{c @{\hspace{-2.6ex}} c @{\hspace{-2.6ex}} c}
        \begin{tabular}{c}
        \includegraphics[width=.32\textwidth]{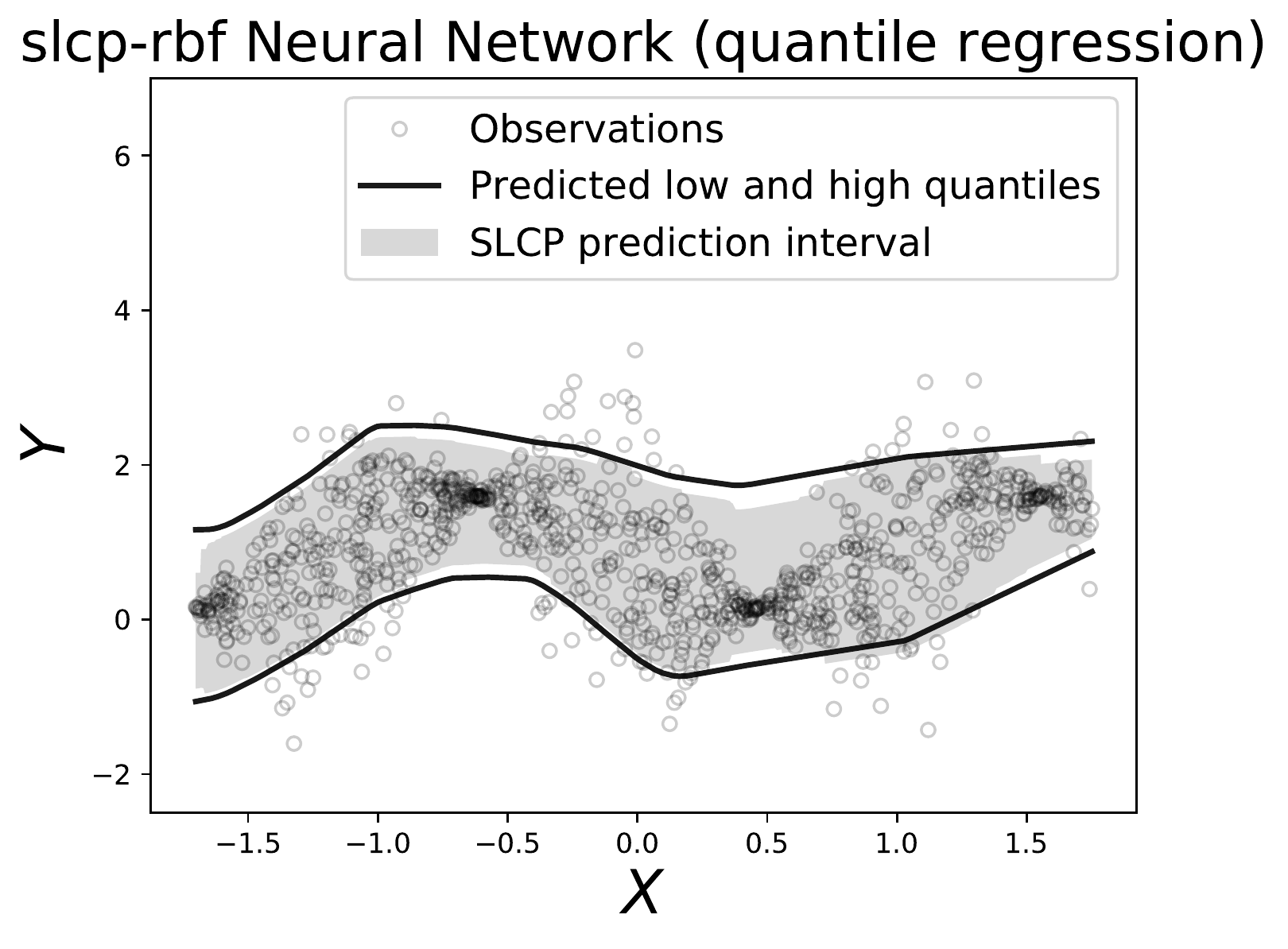}
        \\
        {\small{(a) Avg. cov. 88.96\%; Avg. len. 1.96}}
        \end{tabular} &
        \begin{tabular}{c}
        \includegraphics[width=.32\textwidth]{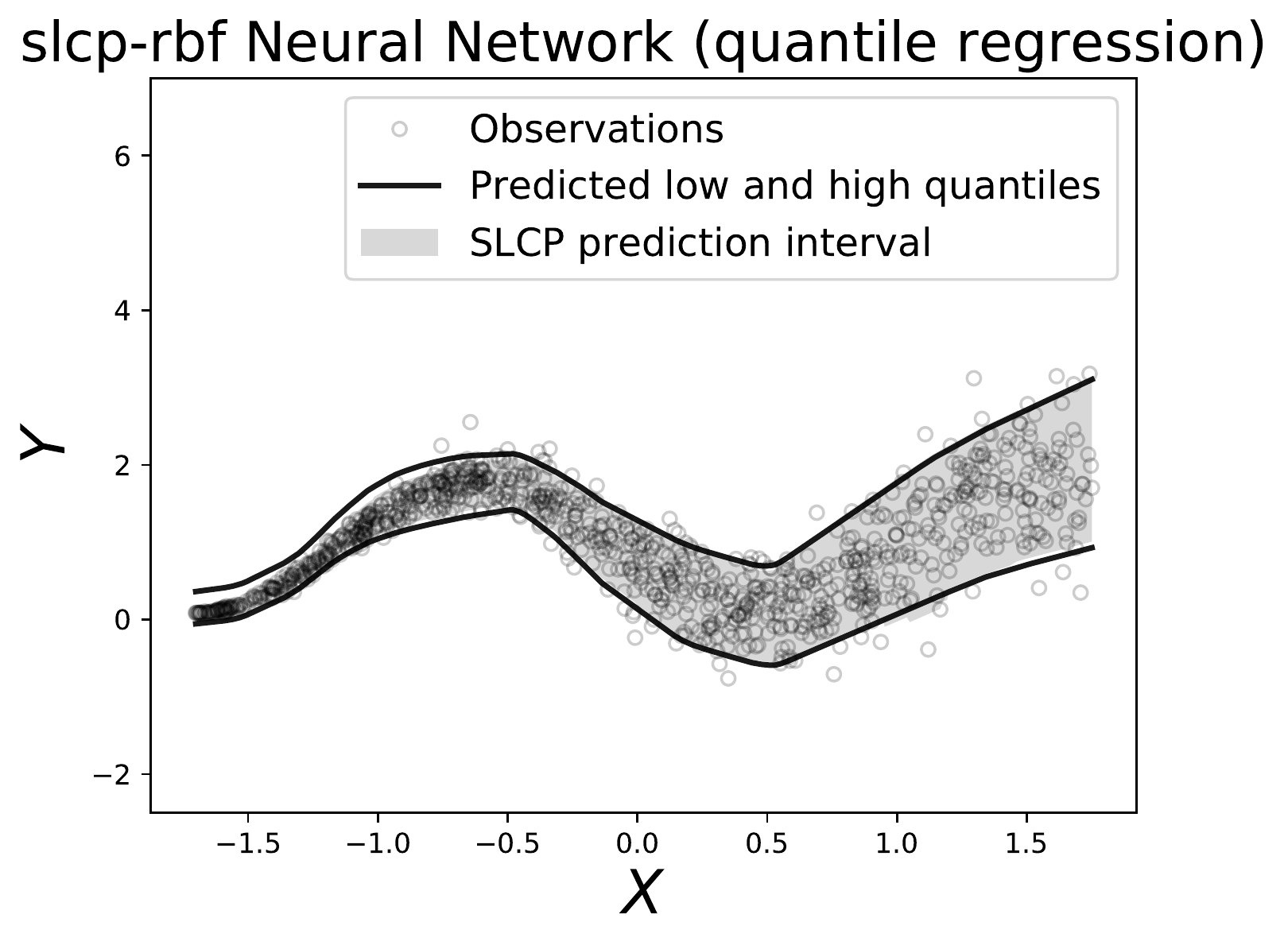}
        \\
        {\small{(b) Avg. cov. 89.00\%; Avg. len. 1.05}}
        \end{tabular} &
        \begin{tabular}{c}
        \includegraphics[width=.32\textwidth]{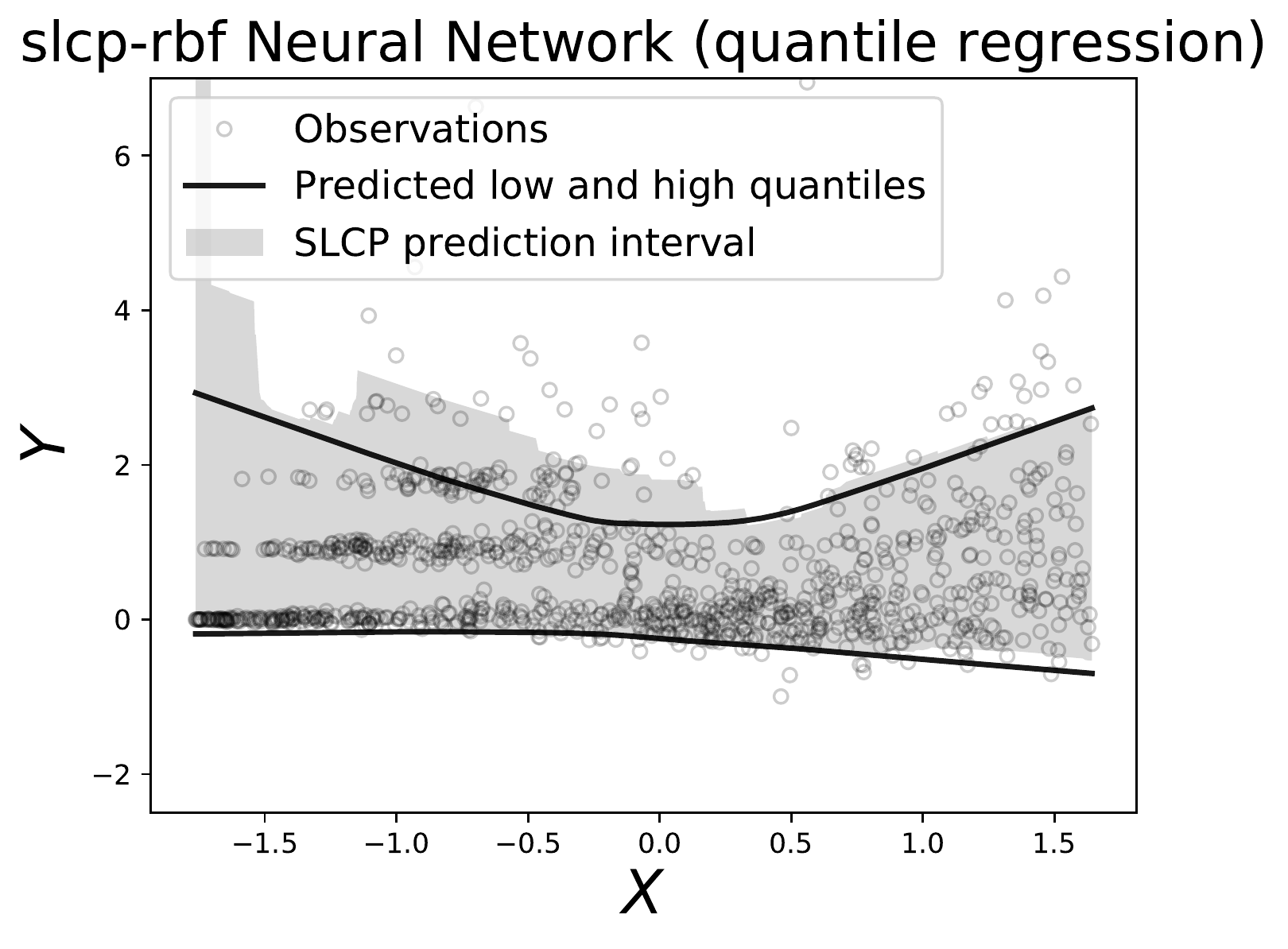}
        \\
        {\small{(c) Avg. cov. 90.44\%; Avg. len. 2.83}}
        \end{tabular} \\
        \end{tabular}
    \end{minipage}
    \vspace{-2ex}
    \caption{SLCP-EPA on quantile neural network estimator.}
    \label{fig:slcp-rbf_neural_net_sim}
\end{figure*}

\begin{figure*}[t!]
    \begin{minipage}{\textwidth}
    \centering
    \begin{tabular}{c @{\hspace{-2.6ex}} c @{\hspace{-2.6ex}} c}
        \begin{tabular}{c}
        \includegraphics[width=.32\textwidth]{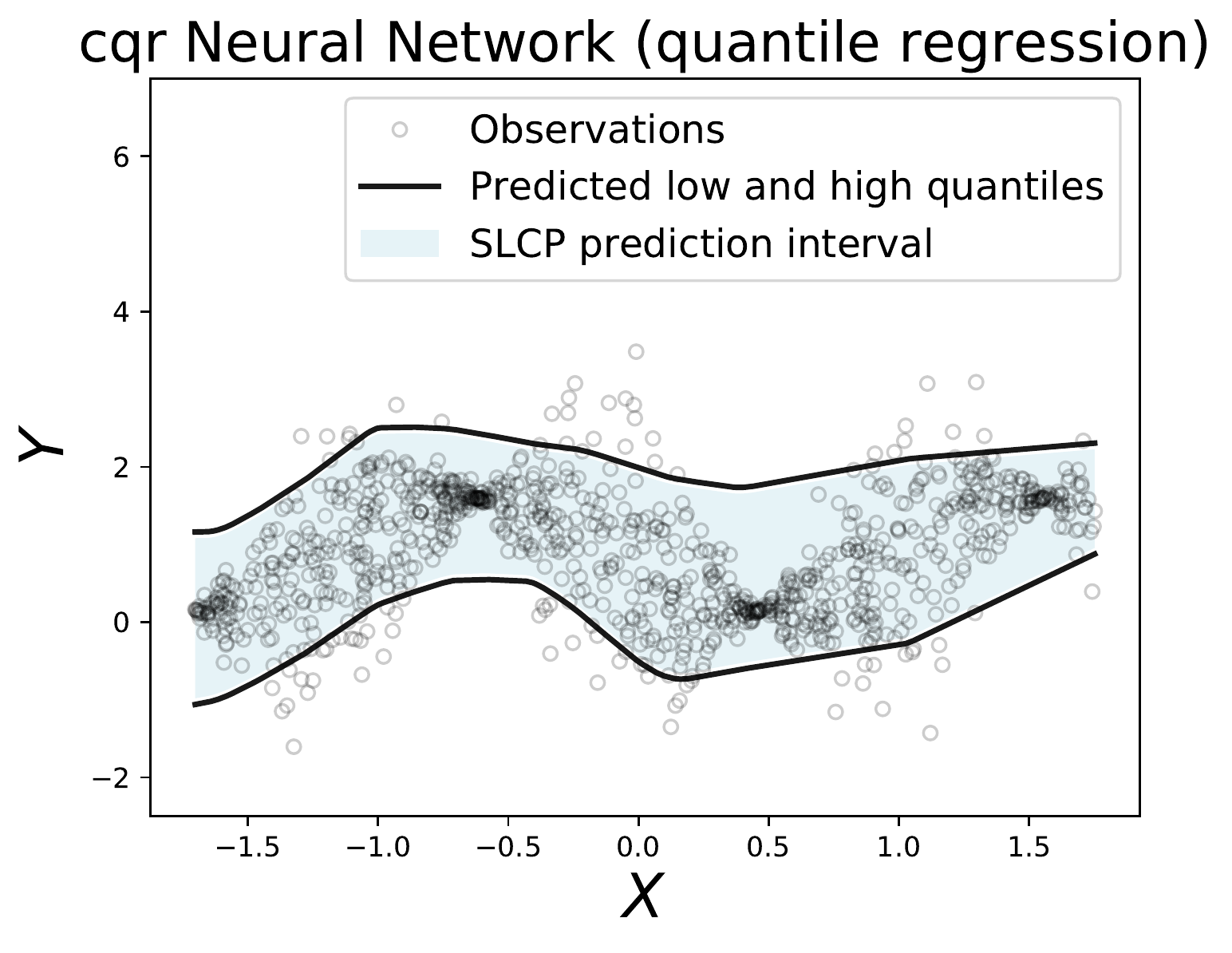}
        \\
        {\small{(a) Avg. cov. 88.88\%; Avg. len. 2.01}}
        \end{tabular} &
        \begin{tabular}{c}
        \includegraphics[width=.32\textwidth]{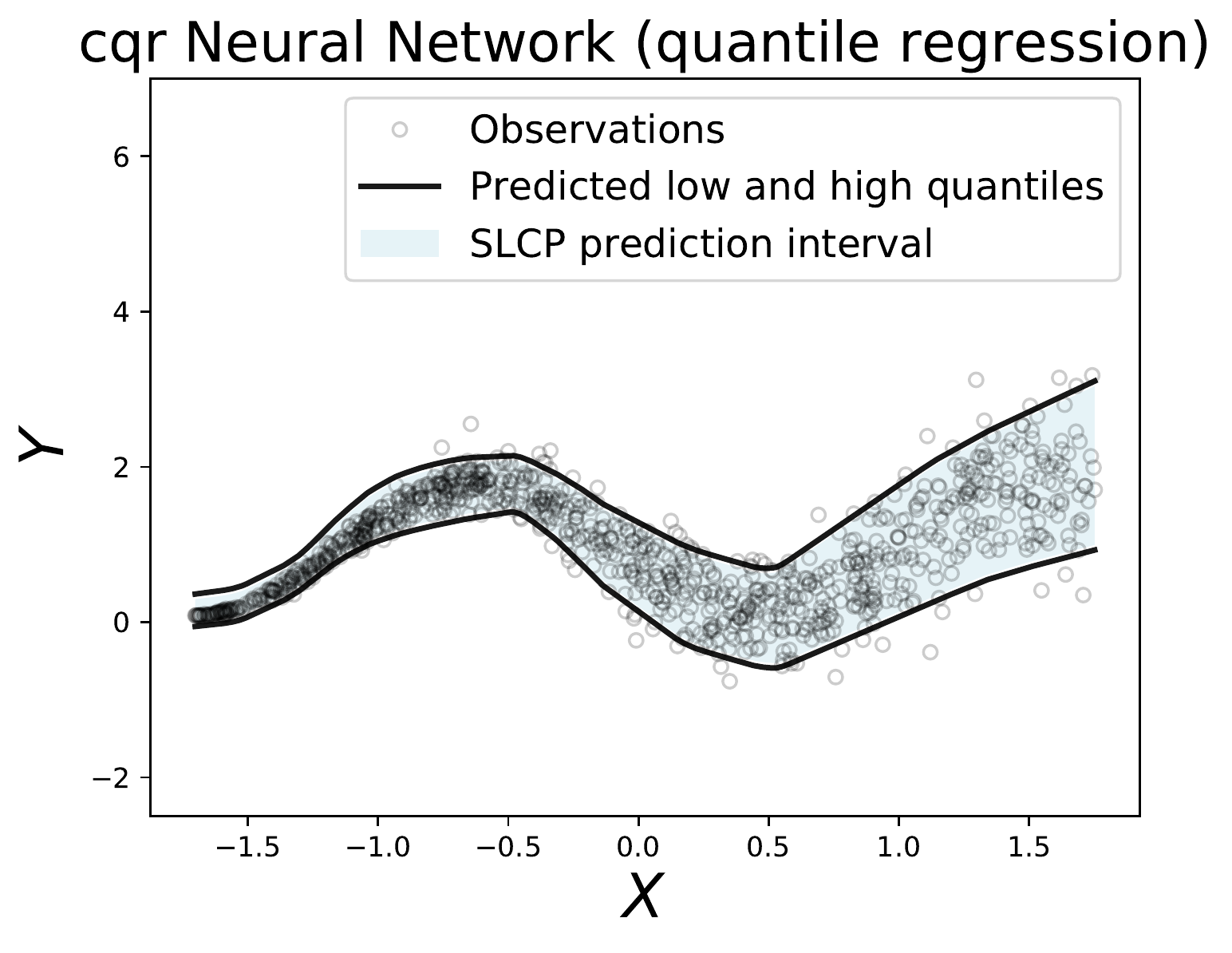}
        \\
        {\small{(b) Avg. cov. 88.80\%; Avg. len. 1.02}}
        \end{tabular} &
        \begin{tabular}{c}
        \includegraphics[width=.32\textwidth]{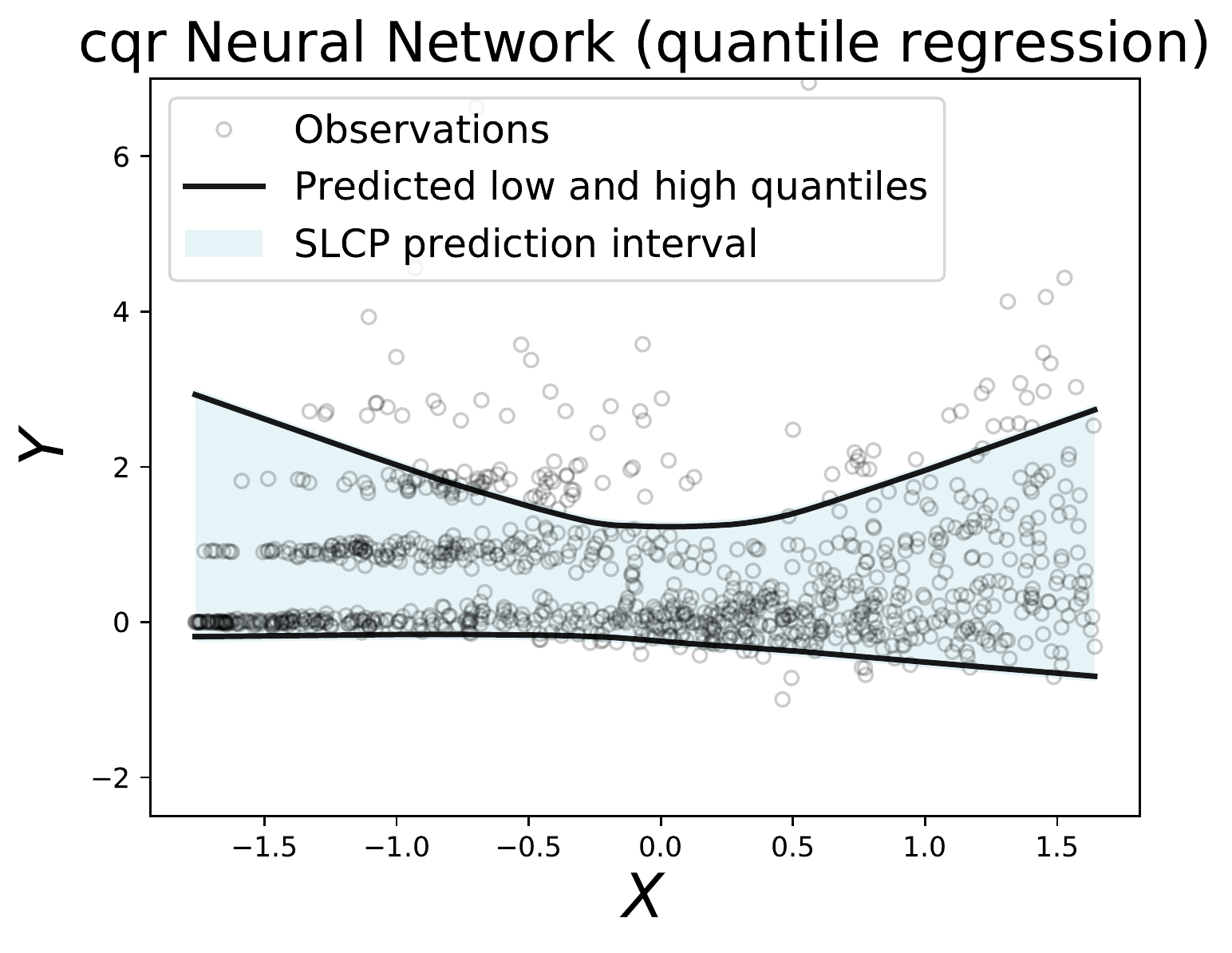}
        \\
        {\small{(c) Avg. cov. 90.34\%; Avg. len. 2.32}}
        \end{tabular} \\
        \end{tabular}
    \end{minipage}
    \vspace{-2ex}
    \caption{SLCP-RBF on quantile neural network estimator.}
    \label{fig:cqr_neural_net_sim}
\end{figure*}

\begin{table*}[t]
\centering
\renewcommand
\arraystretch{1.2}
\scalebox{0.95}{
\begin{tabular}{c|c|c|c|c|c|c}
\hlinewd{1.5pt}
& & SCLP & CQR & LVD & LCP & CD-Split \\ \hline
\multirow{2}{*}{ Overall } & Ave. Length & \textbf{0.20} & \textbf{0.20} & 0.28 & 0.24 & 0.22 \\
& Ave. Coverage & 90.03 & 89.48 & 90.36 & 89.45 & 90.91 \\ \hline
\multirow{2}{*}{ Caucasian (1653) } & Ave. Length & \textbf{0.22} & 0.23 & 0.27 & 0.24 & 0.26 \\
& Ave. Coverage & 89.75 & 89.6 & 90.02 & 89.63 & 90.3 \\ \hline
\multirow{2}{*}{ African American (498) } & Ave. Length & 0.16 & 0.18 & 0.31 & 0.23 & \textbf{0.15} \\
& Ave. Coverage & 91.12 & 90.65 & 91.41 & 89.37 & 92.46 \\ \hline
\multirow{2}{*}{ Asia (7) } & Ave. Length & \textbf{0.26} & \textbf{0.26} & 0.33 & 0.28 & \textbf{0.26} \\
& Ave. Coverage & \textcolor{red}{85.71} & \textcolor{red}{71.41} & 100 & 100 & \textcolor{red}{85.71} \\ \hline
\multirow{2}{*}{ Hispanic (2) } & Ave. Length & 0.31 & 0.31 & 0.35 & 0.31 & \textbf{0.26} \\
& Ave. Coverage & 100 & \textcolor{red}{0} & \textcolor{red}{50} & 100 & \textcolor{red}{0} \\ \hline
\multirow{2}{*}{ Other (1) } & Ave. Length & \textbf{0.23} & 0.24 & 0.27 & 0.24 & 0.26 \\
& Ave. Coverage & 100 & \textcolor{red}{0} & 100 & \textcolor{red}{0} & 100 \\
\hlinewd{1.5pt}
\end{tabular}
}
\caption{Conditional coverage results of STAR dataset; the target coverage rate is 90\%. Best results are highlighted in bold font and results failed to meet the desired coverage rate are highlighted in red.}
\label{tab:star}
\end{table*}

\begin{table*}[t]
\centering
\renewcommand
\arraystretch{1.2}
\scalebox{0.95}{
\begin{tabular}{c|c|c|c|c|c|c}
\hlinewd{1.5pt}
& & SCLP & CQR & LVD & LCP & CD-Split \\ \hline
\multirow{2}{*}{ Overall } & Ave. Length & \textbf{0.54} & 0.57 & 0.81 & 0.75 & 0.58 \\
& Ave. Coverage & 89.65 & 89.69 & 90.44 & \textcolor{red}{88.72} & 89.30 \\ \hline
\multirow{2}{*}{ Clear (7192) } & Ave. Length & \textbf{0.51} & 0.52 & 0.79 & 0.76 & 0.52 \\
& Ave. Coverage & 90.41 & 90.45 & 90.79 & \textcolor{red}{88.94} & 91.00 \\ \hline
\multirow{2}{*}{ Mist (2834) } & Ave. Length & \textbf{0.63} & 0.73 & 0.85 & 0.72 & 0.78 \\
& Ave. Coverage & \textcolor{red}{88.42} & 89.43 & 89.86 & \textcolor{red}{88.23} & \textcolor{red}{87.46} \\ \hline
\multirow{2}{*}{ Snow (859) } & Ave. Length & 0.59 & 0.58 & 0.84 & 0.79 & \textbf{0.54} \\
& Ave. Coverage & \textcolor{red}{84.29} & \textcolor{red}{83.89} & \textcolor{red}{86.04} & \textcolor{red}{82.41} & \textcolor{red}{80.55} \\
\hlinewd{1.5pt}
\end{tabular}}
\caption{Conditional coverage results of bike dataset; the target coverage rate is 90\%. Best results are highlighted in bold font and results failed to meet the desired coverage rate are highlighted in red.}
\label{tab:bike}
\end{table*}

\begin{table*}[t]
\centering
\renewcommand
\arraystretch{1.2}
\scalebox{0.95}{
\begin{tabular}{c|c|c|c|c|c|c}
\hlinewd{1.5pt}
& & SCLP & CQR & LVD & LCP & CD-Split \\ \hline
\multirow{2}{*}{ Overall } & Ave. Length & \textbf{2.33} & 2.39 & 3.41 & 2.98 & 2.37 \\
& Ave. Coverage & 90.23 & 90.64 & 91.45 & 89.58 & 90.10 \\ \hline
\multirow{2}{*}{ Region 1 (997) } & Ave. Length & \textbf{2.54} & 2.72 & 3.64 & 3.15 & 2.57 \\
& Ave. Coverage & 89.07 & 89.17 & 90.22 & \textcolor{red}{88.43} & \textcolor{red}{87.66} \\ \hline
\multirow{2}{*}{ Region 2 (1232) } & Ave. Length & \textbf{2.42} & 2.60 & 3.32 & 3.11 & 2.46 \\
& Ave. Coverage & \textcolor{red}{87.98} & \textcolor{red}{88.23} & 89.74 & \textcolor{red}{86.92} & \textcolor{red}{88.39} \\ \hline
\multirow{2}{*}{ Region 3 (2368) } & Ave. Length & 2.24 & \textbf{2.15} & 3.42 & 2.86 & 2.32 \\
& Ave. Coverage & 91.45 & 91.77 & 92.66 & 90.21 & 91.3 \\ \hline
\multirow{2}{*}{ Region 4 (1666) } & Ave. Length & \textbf{2.20} & 2.34 & 3.29 & 2.82 & 2.27 \\
& Ave. Coverage & 90.86 & 91.72 & 91.84 & 90.13 & 91.12 \\
\hlinewd{1.5pt}
\end{tabular}}
\caption{Conditional coverage results of meps-21 dataset; the target coverage rate is 90\%. Best results are highlighted in bold font and results failed to meet the desired coverage rate are highlighted in red.}
\label{tab:meps}
\end{table*}

\end{document}